\documentclass[twoside]{article}
\usepackage[accepted]{aistats2015}

\usepackage{times}
\usepackage{graphicx} 

\usepackage{amsthm,amsmath,amssymb}

\numberwithin{equation}{section}
\theoremstyle{remark}

\numberwithin{equation}{section}

\theoremstyle{plain}

\newtheorem{theorem}{Theorem}[section]

\theoremstyle{definition}
\newtheorem{definition}{Definition}[section]
\newtheorem{assumption}{Assumption}[section]

\usepackage{algorithm}
\usepackage{algorithmic}

\usepackage{hyperref}

\begin{document}

%

%

\twocolumn[

\aistatstitle{Implementable confidence sets in high dimensional regression}

\aistatsauthor{ Alexandra Carpentier}

\aistatsaddress{Statistical Laboratory, Center for Mathematical Sciences - University of Cambridge } ]

\begin{abstract}
  We consider the setting of linear regression in high dimension. We focus on the problem of constructing \textit{adaptive and honest confidence sets} for the sparse parameter $\theta$, i.e.~we want to construct a confidence set for $\theta$ that contains $\theta$ with high probability, and that is as small as possible. The $l_2$ diameter of a such confidence set should depend on the sparsity $S$ of $\theta$ - the larger $S$, the wider the confidence set. However, in practice, $S$ is unknown. This paper focuses on constructing a confidence set for $\theta$ which contains $\theta$ with high probability, whose diameter is adaptive to the unknown sparsity $S$, and which is implementable in practice.
\end{abstract}

\section{Introduction}

We consider the regression model in dimension $p$, with $n$ observations,
\begin{equation}\label{eq:model}
Y = X \theta + \epsilon,
\end{equation}
where $Y$ is the $n$-dimensional observation vector, $\theta$ is the $p$ dimensional unknown parameter, $X$ is the $n \times p$ design matrix, and $\epsilon$ is the $n$ dimensional white noise (see Section~\ref{sec:setting} for a complete presentation of the setting). We focus on the high dimensional setting where $n \ll p$.

Models such as the one described in Equation~\eqref{eq:model} have received very much attention recently. In particular, finding good estimates of $\theta$ when $p$ is very large has many important applications (see~\cite{starck2010sparse,NIPS2011_0141,kavukcuoglu2009learning}). Solving this problem in a satisfying way is nevertheless impossible in general, since it is ill-posed. For this reason, and also because it is an assumption that holds in many concrete cases, it is usual in this setting to focus on the case where $\theta$ is a sparse parameter. Let $l_0(S)$ be the $l_0$ ``ball" of radius $S$, i.e.~the set of vectors that have less than $S$ non-zero coordinates (that are $S$-sparse). It has been proved that in some specific cases, namely when the design matrix $X$ satisfies some desirable conditions for $\bar p$ sparse vectors (e.g.~null space property, R.I.P., restricted eigenvalue condition, etc, see e.g.~\cite{candes2006robust, koltchinskii2009dantzig, donoho1989uncertainty, candes2008restricted, foucart2009sparsest, bickel2009simultaneous}), it is possible to construct an estimate $\hat \theta(X,Y)$ of the parameter $\theta$ such that
\begin{equation}\label{eq:adaptinf}
\sup_{S: S \leq \bar p} \sup_{\theta \in  l_0(S)} \mathbb P_{\theta}\Big( \|\hat \theta - \theta\|_2^2 \geq E\frac{S \log(p/\delta)}{n}\Big)<\delta,
\end{equation}
where $E>0$ is a constant, where for any vector $u$, $\|u\|_2^2 = \sum_j u_j^2$ is the usual $l_2$ norm, and where $\mathbb P_{\theta}$ is the probability according to the noise $\epsilon$ when the true parameter is $\theta$. This result is minimax-optimal over $S-$sparse vectors for any $S \leq \bar p$, see e.g.~\cite{raskutti2011minimax}. Moreover, this bound on the accuracy of $\hat \theta$ scales with the true sparsity $S$ of $\theta$ that is not available to the learner : the estimate $\hat\theta$ is adaptive. Some key references for the existence of such an adaptive estimate are~\cite{zou2006adaptive,candes2007dantzig, bickel2009simultaneous, buhlmann2011statistics}. Although this problem is a difficult combinatorial problem, there exist some computationally feasible techniques, under stronger assumptions on the design $X$, for instance the thresholding procedures, the orthogonal matching pursuit, the Lasso, the Dantzig Selector etc. For more references on the techniques and the associated bounds and design assumptions, see e.g.~\cite{donoho1989uncertainty, Tibshirani94regressionshrinkage, donoho2006compressed, candes2006robust, lee2013exact}.

Another important problem is the one of confidence statements for the parameter $\theta$, i.e.~of quantifying the precision of an estimate of $\theta$. Constructing confidence sets in high dimensional regression was studied e.g.~in the papers~\cite{abbasi2012,javanmard2013confidence,beran1998modulation, baraud2004confidence, nickl2012confidence} and it is, with the problem of estimating $\theta$, the second fundamental problem of inference in this setting. The objective is to construct a set $C_n$ that contains $\theta$ with high probability, and also that is as small as possible, i.e.~that is such that its $l_2$ diameter is as small as possible. One can deduce from the lower and upper bounds for the estimation problem (see~\cite{raskutti2011minimax, nickl2012confidence, javanmard2013confidence}), that the optimal $l_2$-width of a confidence interval for the sparse vector $\theta$ should depend on its sparsity $S$ - it should be of order $\sqrt{S\log(p)/n}$. For $\delta>0$, if $\theta$ is $S$ sparse and \textit{$S$ is known}, and if $\hat \theta$ is an estimator of $\theta$ that satisfies Equation~\eqref{eq:adaptinf}, a $l_2$-confidence interval $C_n$ of coverage $1-\delta$ should ideally be of the form
\begin{equation*}
C_n = \{u \in \mathbb R^p: \|u - \hat \theta\|_2 \leq \sqrt{E \frac{S \log(p/\delta)}{n}}\}.
\end{equation*}
On the other hand, if the sparsity $S$ of the parameter $\theta$ is unknown, which is the case in real applications since $\theta$ is unknown, one cannot construct directly this optimally sized confidence interval. 

For the problem of estimating $\theta$, \textit{computationally feasible} techniques that are \textit{adaptive to the unknown sparsity $S$} of $\theta$ exist (see Equation~\eqref{eq:adaptinf} and associated references). Do similar results hold for the problem of constructing a confidence set for $\theta$? In particular, can one construct, in a non computationally extensive way, a confidence set for the adaptive estimate of Equation~\eqref{eq:adaptinf} that contains $\theta$ with high probability and whose diameter is adaptive to the unknown sparsity $S$ of $\theta$? This is the problem that this paper targets. 
We would like to insist on the importance of the construction of confidence sets. Indeed, most of the sequential learning algorithms rely on such confidence sets. For instance, in the papers~\cite{abbasi2012,carpentier2012,desphande2012} that are on the topic of linear bandit in high dimensions, the authors assume that an upper bound $\bar S$ on the sparsity $S$ is known, and they consider a large confidence interval for $\theta$ whose diameter depends on $\bar S$. The final bounds on the regret of their bandit algorithms then depend on the chosen upper bound $\bar S$, and not on the correct sparsity $S$ of $\theta$. In this setting, it would be quite useful to have an adaptive and honest confidence set for $\theta$, that adapts to the sparsity $S$. The bound on the regret would then depend on $S$ and not on $\bar S$.


\section{Literature review}\label{lit:rev}

The problem of constructing a confidence set, when the support of the parameter or its sparsity $S$ is known, is a problem that has received attention recently, see~e.g.\cite{abbasi2012, javanmard2013confidence, javanmard2014confidence, beran1998modulation, baraud2004confidence, lee2014, VandeGeer}. The papers~\cite{javanmard2013confidence, javanmard2014confidence, lee2014, VandeGeer} are concerned with finding the limiting distribution of an estimate of $\theta$, and using it to build tests and confidence sets for fixed, low dimensional sub models (fixed subsets of coordinates of $\theta$). This approach does not provide an optimally sized confidence set for the parameter $\theta$ itself, since its support (i.e.~the correct low dimensional model of interest) is not known.
On the other hand, the problem of constructing an adaptive and honest confidence interval for $\theta$ when $S$ is unknown, has only recently started to receive attention. It is an important problem, since there is no reason why the low dimensional support of the parameter has to be known before hand : therefore, the low dimensional model of interest cannot be chosen efficiently in a non data driven way. This problem is related to the problem of estimating the sparsity $S$ of the parameter $\theta$, as explained in various related settings in the papers~\cite{hoffman2002random, juditsky2003nonparametric, gine2010confidence, nickl2012confidence}. Indeed, if a good estimate $\hat S$ of $S$ is available, then one could consider the confidence interval
\begin{equation*}
C_n = \{u \in \mathbb R^p: \|u - \hat \theta\|_2^2 \leq E \frac{\hat S \log(p/\delta)}{n}\}.
\end{equation*}

Let us first consider a simpler instance of this problem that will enlighten its difficulties. It is the case where one wants to be adaptive to two sparsities $S_0<S_1$ (and not to any sparsity $S$). The objective is to construct a confidence set $C_n$ that is \textit{adaptive and honest}, i.e.~that contains $\theta$ with high probability, and whose diameter is of order $\sqrt{S_0 \log(p)/n}$ if $\theta$ is of sparsity $S_0$ or below, or of order $\sqrt{S_1 \log(p)/n}$ if the sparsity of $\theta$ is between $S_0$ and $S_1$. In other words, the objective is to construct a set $C_n$ based on the data such that for $\mathcal P= l_0(S_1)$, for $\mathcal I=\{S_0, S_1\}$, and for $\delta>0$,
\begin{align}
&\max_{S \in \mathcal I} \sup_{\theta \in l_0(S)\bigcap \mathcal P}\mathbb P_{\theta}(\theta \in C_n) \geq 1-\delta, \hspace{0.2cm} \mathrm{and} \hspace{0.2cm} \nonumber\\
&\max_{S \in \mathcal I} \sup_{\theta \in  l_0(S) \bigcap \mathcal P} \mathbb P_{\theta}\Big(|C_n|_2 \geq \sqrt{E' \frac{S \log(p/\delta)}{n}}\Big) \leq \delta,\label{eq:csetint2}
\end{align}
where $E'>0$ is a constant and where $|C_n|_2 = \sup\{\|u - v\|_2,  (u,v)\in C_n^2\}$ is the diameter of $C_n$. 
There is however a serious obstruction to the creation of a such confidence set. It is possible to prove (see e.g.~\cite{baraud2004confidence, nickl2012confidence}) that in many situations, there exists no adaptive and honest confidence sets on the entire parameter space $\mathcal P = l_0(S_1)$. The problem is that there are some parameters that are not $S_0$ sparse, but that are very close to $S_0$ sparse vectors, and for which it is impossible to detect that one needs a confidence set of diameter $\sqrt{E'S_1 \log(p/\delta)/n}$ (since a confidence set of diameter $\sqrt{E'S_0 \log(p/\delta)/n}$ won't provide enough coverage). A reasonable and important question is then to provide a confidence set that is adaptive and honest on the largest possible model $\mathcal P \subset l_0(S_1)$. Intuitively, this model $\mathcal P$ should be the set $l_0(S_1)$ where the problematic parameters that are not $S_0$ sparse, but are very close to $l_0(S_0)$ have been removed.

There has been recently some very important advances on this problem 
in the paper~\cite{nickl2012confidence}. They define the separated set $\tilde {l_0} (S_1, \rho)$ for a constant $\rho >0$ as
\begin{equation}\label{eq:sepsets}
\tilde {l_0} (S_1, \rho) = \{u \in l_0(S_1): \|u - l_0(S_0)\|_2 \geq \rho\},
\end{equation}
where, if $A \subset \mathbb R^p$, we write for $u \in \mathbb R^p$, $\|u - A\|_2 = \inf_{v \in A} \|u-v\|_2$. They then define
\begin{equation*}
\mathcal P: = \mathcal P(\rho):= l_0(S_0) \bigcup \tilde {l_0} (S_1, \rho).
\end{equation*}
This new model excludes vectors that are not $S_0$ sparse, but at a distance that is less than $\rho$ from $S_0$ sparse vectors. The smaller $\rho$, the more similar $\mathcal P$ is to $l_0(S_1)$ (equality holds when $\rho=0$). The restriction to the model $\mathcal P$ can be seen as a margin condition : the $\rho-$margin condition is satisfied if the true parameter $\theta$ belongs to a sub-model where the two classes of sparsity are $\rho-$away of each other, i.e.~if $\theta$ belongs either to $l_0(S_0)$, or to $\tilde {l_0} (S_1, \rho)$. This margin condition is necessary for being able to distinguish between the two sets of sparsity $S_0$ and $S_1$.

The objective is then to characterize the smallest possible $\rho$ for which a such confidence set exists, and then to construct this confidence set. 
Table~\ref{fig:1} summarizes the minimax-optimal order of $\rho:=\rho_n$ (with lower and upper bounds) such that an adaptive and honest confidence set for $\theta$ exists in $\mathcal P(\rho_n)$, in function of $S_0,S_1$ (see~\cite{nickl2012confidence}). All upper bounds and lower bounds match in the three cases summarized in the table, and $\lim_{n\rightarrow \infty}\rho_n = 0$ at a rate which depends on $S_0,S_1$, which implies that $\mathcal P(\rho_n)$ converges to $l_0(S_1)$. 
The minimax-optimal rates at which $\rho_n$ can go to $0$ are known for this problem, but an important issue that remain on the existence of \textit{computationally feasible} adaptive and honest confidence intervals in cases (ii) and (iii) of Table~\ref{fig:1}. Indeed, the procedure in~\cite{nickl2012confidence} consists on computing a quantity of the form $\inf_{u \in l_0(S_0)} |t_n(u)|$ where $t_n(u)$ is some quadratic functional of the data and $u$, then testing if this quantity exceeds a threshold, and finally using the output of this test for constructing the confidence interval. Computing this statistic is a computationally extensive problem, since its computational complexity is of order $p^{S_0}n$. It is thus not proper for concrete applications. Also, they assume that their design matrix is a random Gaussian matrix, which is quite restrictive. Another aspect that moreover prevents the use of this method in practical applications is the fact that sparsity is unlikely to hold exactly in practice, and that results involving approximate sparsity are more appealing.


\begin{table*}[t]
\begin{center}
\begin{tabular}{|c||ccc|}
  \hline
& \textbf{Case (i)} & \textbf{Case (ii)} & \textbf{Case (iii)}\\
\textbf{Values of $S_0$ and $S_1$} & $S_0 = o(n^{1/2}/\log(p))$ & $S_0 = o(n^{1/2}/\log(p))$ & $S_0 = o(n/\log(p))$\\
&   $S_1 =o(n^{1/2}/\log(p))$ &  $S_1 =o(n/\log(p))$ &   $S_1 =o(n/\log(p))$ \\
\hline
\hline
  UB on $\rho$ & $\sqrt{\frac{S_1 \log(p)}{n}}$  & $n^{1/4}$ &  $0$\\
  \hline
  LB on $\rho$ & $\sqrt{\frac{S_1 \log(p)}{n}}$  & $n^{1/4}$ & $0$\\
  \hline
  Computational complexity & $np$ & $np^{S_0}$ & $np^{S_0}$ \\
  \hline
\end{tabular}
\end{center}
\caption{Upper and lower bounds on the parameter $\rho^2$ for the problem of constructing an honest and adaptive confidence set. UP stands for Upper Bound, LB stands for Lower Bound.}\label{fig:1}
\end{table*}


In this paper, we provide a computationally feasible method for constructing an honest and adaptive confidence set on a maximal model $\mathcal P$ (in the minimax-optimal sense of Table~\ref{fig:1}), in a more general setting, i.e.~in the setting of approximate sparsity for the parameter $\theta$, and under more general assumptions for the design matrix $X$, namely that its condition number is not too large for $S_1$ sparse vectors. The confidence set we propose is actually trivial to implement and its computational complexity is of order $O(pn)$ (provided that an adaptive estimate satisfying Equation~\ref{eq:adaptinf} is available). We first provide a method in the case of two sparsity indexes (that achieves adaptivity to two sparsities $S_0,S_1$), and we then extend it to the more general setting of many sparsity indexes $\mathcal I$. In particular, the method we propose applies to constructing a confidence set that is adaptive and honest for all the sparsities smaller than a large sparsity index $\bar p$. It is minimax optimal also in this setting, adaptive, and implementable in practice. We test our method with numericals simulations and all proofs are in the supplementary material.

\section{Setting}\label{sec:setting}

Let $n$ and $p$ be two integers with $n \ll p$. Consider the model defined in Equation~\eqref{eq:model}.

\subsection{Assumption on the noise $\epsilon$}

We state the following assumption about the noise $\epsilon$, namely that its entries are independent and sub-Gaussian.
\begin{assumption}\label{ass:noise}
The entries of $\epsilon$ are independent. Moreover, $\forall i \leq n, \mathbb E \epsilon_i =0, \mathbb V \epsilon_i = \sigma^2$ and $\forall i \leq n, \forall \lambda >0$, there exists $c>0$ such that
\begin{equation*}
\mathbb E \exp(-\lambda \epsilon_i) \leq \exp(-\lambda^2 c^2/2).
\end{equation*}
\end{assumption}
For instance bounded random variables and Gaussian random variables satisfy this Assumption.

\subsection{Assumption on the design matrix $X$}

We make the following assumption about the design matrix $X$.
\begin{assumption}\label{ass:design}
Let $\bar p>0$. The matrix $X$ is such that there exists two constants $C_M>c_m>0$ such that for any $u$ that is $\bar p-$sparse, we have
\begin{equation*}
c_m\|u\|_2 \leq \|\frac{1}{\sqrt{n}}Xu\|_2 \leq C_M\|u\|_2.
\end{equation*}
\end{assumption}
This assumption is a relaxation of the celebrated R.I.P. condition (see~\cite{foucart2009sparsest} for another paper in this setting). This condition makes sense, when $n \ll p$, only if $\bar p$ is actually smaller than $n$, i.e.~$\bar p = O(n/\log(p))$. In this case, for instance, random Fourier matrices, and more generally RIP matrices, satisfy this condition, with $c_m$ and $C_M$ close to $1$. More generally, e.g.~random matrices with sub-Gaussian entries whose variance-covariance matrix has bounded condition number satisfy this condition for $\bar p = O(n/\log(p))$ and $c_m$ and $C_M$ depending on the condition number of the variance-covariance matrix. 

\subsection{The set of vectors of approximate sparsity}

We are interested in situations where $\theta$ is \textit{approximately} $S$-sparse. More specifically, we focus on vectors $\theta$ that have less than $S$ ``large" components, but that can have up to $\bar p$ ``small" components such that their $l_2$ norm is not too large.
\begin{definition}\label{def:app:spa}
We define the following sets of approximately $S-$sparse vectors, for $B,C,\bar p,\delta>0$, as 
\begin{align*}
\mathcal S_S(C,B, \bar p, \delta) &= \Big\{u  \in l_0(\bar p), \|u\|_2\leq B:\\ 
&\|u - l_0(S)\|_2^2 \leq \frac{CS\log(p/\delta)}{n} \Big\},
\end{align*}
where for a vector $u$, $\|u - l_0(S)\|_2^2 = \inf_{v \in l_0(S)} \|u - v\|_2^2 = \sum_{j=S+1}^p u_{(j)}^2$, where $u_{(.)}$ is the ordered version of $|u|$, i.e.~is such that $|u_{(1)}| \geq |u_{(2)}| \geq ... \geq |u_{(p)}|$. 
\end{definition}
The vectors in these sets have at most $S$ ``large" components, and the $p-S$ remaining components have small $l_2$ norm. An important property of $\mathcal S_S(C,B, \bar p, \delta)$ is that it contains all $S-$sparse vectors whose $l_2$ norm is bounded by $B$, and is an enlargement of the set of sparse vector to ``approximately" sparse vectors. 


Let $0<S_0<S_1$ be two sparsities. Similarly to what is proposed in Equation~\eqref{eq:sepsets}, we define the separated set as
\begin{align}\label{eq:sepsets2}
&\tilde {\mathcal S}_{S_1} (C,B, \bar p, \delta, \rho) := \tilde {\mathcal S}_{S_1,S_0} (C,B, \bar p, \delta, \rho) \nonumber\\
&= \{u \in \mathcal S_{S_1}(C,B, \bar p, \delta): \|u - l_0(S_0)\|_2 \geq \rho\}.
\end{align}
These sets are such that, between $\mathcal S_{S_0}$ and $\tilde {\mathcal S}_{S_1}$, there is a $\rho$-margin condition.

For the same value of $\rho$, these sets are strictly larger than the sets presented in Equation~\eqref{eq:sepsets} with bounded $l_2$ norm, which are actually the sets considered in paper~\cite{nickl2012confidence}. Indeed, they correspond to the vectors in the enlarged sets $\mathcal S_{S_1}(C,B, \bar p, \delta)$ that are at least $\rho$-away from $l_0(S_0)$. 

\subsection{Adaptive and honest confidence sets}

We now provide the definition of adaptive and honest confidence sets in the wider model of approximately sparse vectors. It is an extension of the definition provided in Equation~\eqref{eq:csetint2} to the larger set of approximately sparse vectors (it demands that the second equation in Definition~\eqref{def:csetint} holds also for approximately sparse vectors).
\begin{definition}\label{def:csetint}
Let $\delta, C, B, \bar p>0$. A set $C_n$  is a $\delta-$adaptive and honest confidence set for $\mathcal P  \subset   \mathcal S_{S_1}(C,B, \bar p, \delta)$ and for $\mathcal I \subset \{1, \ldots, \bar p\}$ if
\begin{align*}
&\max_{S \in \mathcal I} \sup_{\theta \in  \mathcal S_{S}(C,B, \bar p, \delta) \bigcap \mathcal P}\mathbb P_{\theta}(\theta \in C_n) \geq 1-\delta, \hspace{0.2cm} \mathrm{and} \hspace{0.2cm}\nonumber\\
&\max_{S \in \mathcal I} \sup_{\theta \in   \mathcal S_{S}(C,B, \bar p, \delta) \bigcap \mathcal P} \mathbb P_{\theta}\Big(|C_n|_2 \geq \sqrt{E' \frac{S \log(p/\delta)}{n}}\Big) \leq \delta,
\end{align*}
\end{definition}
where $E'>0$ is a constant.

\section{Adaptive estimation of $\theta$ on the enlarged sets}\label{sec:est}


We are first going to prove that on these enlarged sets $\mathcal S_S(C,B, \bar p, \delta)$, adaptive inference remains possible, i.e.~that it is possible to build an estimate of $\theta$ that satisfies results similar to what is described in Equation~\eqref{eq:adaptinf}. More precisely, we prove that if the design is not too correlated ($c_m$ and $C_M$ not too far from $1$ in Assumption~\ref{ass:design}) then the lasso estimator will provide good results on the enlarged sets. 
\begin{theorem}[Adaptive Lasso on the enlarged sets]\label{th:adaest}
Let Assumptions~\ref{ass:noise}, and~\ref{ass:design} be satisfied for $c, c_m, c_M, 66\bar p$ such that $c>0, c_m>2/3, C_M<4/3$ and $\bar p >0$. Let $B>0$ and $C>0$. Let $\delta>0$. The solution $\hat \theta$ of $l_1$ minimization or lasso
\begin{align*}
\hat \theta &= \arg\min_u \Big[ \|Y - Xu\|_2^2 +  \kappa \sqrt{\log(p/\delta)n}\|u\|_1\Big],
\end{align*}
where $\kappa> 4\max(c, \sqrt{C}/3, c^2, C/9)$ is such that we have $\forall 0< S \leq \bar p$
\begin{align*}
\sup_{\theta \in \mathcal S_S(C,B, \bar p, \delta) }\hspace{-1mm}&\mathbb P_{\theta}\Big( \|\hat \theta - \theta\|_2^2\\ 
&\geq \Big(12(36\kappa + 36)^2 + C^2 \Big)\frac{ S  \log(p/\delta)}{n}\Big) \leq \delta.
\end{align*}
\end{theorem}
The proof of this theorem is in the Supplementary Material (see Appendix~\ref{app:12}). Proving this bound for the lasso on the enlarged sets is actually very similar to proving it on the set of exactly sparse vectors. An important remark is that the lasso's computational complexity is not high and is well defined, see~\cite{Tibshirani94regressionshrinkage}. As usual, the lasso does not work on too correlated designs, it can be applied when $c_m>2/3$ and $C_M <4/3$. When this is not satisfied, other techniques have to be considered, see e.g.~\cite{foucart2009sparsest}. It is actually possible to prove that for any $0<c_m<C_M$, there exists an estimate that satisfies a result similar to the one in Theorem~\ref{th:adaest}, see~Theorem~\ref{th:adaest2} in the Supplementary Material Appendix~\ref{app:1}. This estimate is however the result of $l_0$ minimization, and is thus computationally extensive (the computational complexity is $np^{S_0}$), and is in practice not implementable whenever $p, S_0$ are too large.


The really nice feature of such a result is that it provides an estimate whose $l_2$-risk is adaptive \textit{uniformly} to the sparsity of \textit{any} vector of the enlarged sparsity class, for any sparsity smaller than $\bar p$. 
The estimate is data driven, but it needs an upper bound $C$ on the amount of which $\theta$ deviates from the sparsity $S$, and also it needs an upper bound $c$ on the parameter that bounds the sub-Gaussian tail of the noise. 

\section{Adaptive and honest confidence sets for $\theta$}\label{sec:test}

We now propose a method that is computationally feasible for constructing an adaptive and honest confidence set for $\theta$. We fist present this method in the case of adaptivity to only two sparsities $S_0<S_1$ (two sparsity indexes method), and then explain how to extend these results to larger sets of sparsities (multi sparsity indexes method). 




\subsection{Presentation of the confidence set for two sparsity indexes $S_0<S_1$}

\paragraph{Construction of the confidence set}

Let $S_0<S_1 <\bar p$. The algorithm for the two sparsity indexes, Algorithm~\ref{alg:11}, contains two main steps. The first step is to construct a test $\Psi_n$ for deciding whether $\theta$ is $S_0$ approximately sparse, or $S_1$ approximately sparse. The test consists in first computing an adaptive estimate $\hat \theta$ of $\theta$, and then on thresholding all non-significant components. Then, the testing decision is based on two factors (i) testing whether the number of non-zero entries of the thresholded estimate is larger than $S_0$ and (ii) testing whether the squared residuals $\|\hat r\|_2^2$ are significant or not. If both these quantities are small enough, the test is accepted, otherwise it is rejected. The outcome of this is the test $\Psi_n$. The second step is to use this information to construct the confidence set $C_n$. Based on $\Psi_n$, the confidence interval $C_n$ will be of diameter of order $\sqrt{\frac{S_0\log(p)}{n}}$ (if $\Psi_n = 0$), or $\sqrt{\frac{S_1\log(p)}{n}}$ (if $\Psi_n = 1$). The procedure is explained in Algorithm~\ref{alg:11}.

\begin{algorithm}[h!]
\caption{Two sparsity indexes confidence set}
 \label{alg:11}
 \begin{algorithmic}
\STATE \textbf{Parameters:} $\delta, S_0, S_1, \sigma^2$
\SET the following quantities, computed on the first half of the samples only, as $\hat B^2:= 3/2 \Big(n^{-1}\sum_{i \leq n} Y_i^2 (1+2\log(1/\delta)) + 2\log(1/\delta)\Big)$,
and $\tau_n := 14|\hat B|\sqrt{n^{-1/2}\log(1/\delta)}+ 381|\hat B|\sqrt{\frac{S_0\log(p/\delta)}{n}}$,
and  $\tau_n' := 330|\hat B| \sqrt{\frac{S_1\log(p/\delta)}{n}}$,
and let $\hat \theta$ be the lasso estimate as in Theorem~\ref{th:adaest}. All these quantities are computed on the first half of the sample, and from now on we only use the second half of the sample.

\SET the residual $\hat r = Y - X \hat \theta$
\SET the test statistic $R_n = \|\hat r\|_2^2 - n \sigma^2$

\SET the test $\Psi_n = 1 -  \mathbf 1\{R_n \leq \tau_n^2\} \mathbf 1 \{  \sum_{j=S_0+1}^p \hat \theta_{(j)}^2 \leq (\tau_n')^2\}$,
where $\hat \theta_{(.)}$ is the ordered version of $|\hat \theta|$

\SET the confidence interval 
\begin{align*}
C_n := &\big\{u \in \mathcal S_{S_1}: \|u - \hat \theta\|_2 \leq  650\sqrt{ \frac{S_0 \log(p)}{n}} \mathbf 1\{\Psi_n = 0\}\\ 
&+ 650\sqrt{ \frac{S_1 \log(p)}{n}} \mathbf 1\{\Psi_n = 1\}\big\}.
\end{align*}
\STATE \textbf{return} $C_n$
\end{algorithmic}
\end{algorithm}
The parameter $\sigma^2$ car be replaced by a consistent estimator of the variance of the noise $\epsilon$ (e.g.~a Bootstrap estimate, a cross validation estimate, etc).





\paragraph{Main result}

The following theorem states that this confidence interval is adaptive and honest (in the sense of Definition~\eqref{def:csetint}) over a large model $\mathcal P$.


\begin{theorem}\label{cor:conf}

Assume that the noise is either Gaussian of variance less than $1$, or bounded by $1$, and assume that the assumptions for convergence of the adaptive lasso, stated in Theorem~\ref{th:adaest}, are satisfied, and that $S_1 \leq \bar p$. Then the confidence set presented in Algorithm~\ref{alg:11} is $\delta$ adaptive and honest for $\mathcal I = \{S_0 ,S_1\}$ and over the model
$$\mathcal P(\rho_n) = \mathcal S_{S_0} (32,\infty, \bar p, \delta) \bigcup \tilde {\mathcal S}_{S_1} (32,\infty,\bar p, \delta,\rho_n),$$
where
\begin{align*}
\rho_n = |\hat B| \min\Big(54\sqrt{\log(1/\delta)} n^{-1/4}, 460\sqrt{\frac{S_1 \log(p/\delta)}{n}}\Big).
\end{align*}
By definition of the enlarged set, this implies in particular that the confidence set is $\delta$ adaptive and honest for $\mathcal I = \{S_0 ,S_1\}$ and over the model
$$\mathcal P(\rho_n) =  l_0(S_0) \bigcup \tilde l_0(S_1, \tilde \rho_n), \quad \mathrm{with} \quad \tilde \rho_n = 2\rho_n.$$
\end{theorem}
This theorem is actually a corollary of a more general result, presented in the Supplementary Material, Appendix~\ref{app:expl}. The proof of this theorem is in the supplementary material (Appendix~\ref{app:2}).

The confidence set is adaptive and honest under the same assumptions that ensure consistency of the lasso estimate. It is quite reasonable that it is so, since the creation of adaptive and honest confidence sets is a strictly more difficult problem than the problem of estimating the parameter (indeed, any point of the confidence set is a good estimate of the parameter). Also, since $\lim_{n\rightarrow \infty} \rho_n = 0$, for any $\theta$ and for $n$ large enough, the confidence set contains $\theta$ with high probability, and its diameter adapts to the sparsity of $\theta$. It is adaptive to the two sparsities $\{S_0, S_1\}$ only and not to the whole spectrum of sparsities, but it already allows to improve many existing learning algorithms by diminishing the size of the confidence interval (by not always setting it to $\sqrt{S_1\log(p)/n}$ independently of $\theta$). Moreover, it computational complexity is of order $np$, which is linear.


\paragraph{Comparison with results in paper~\cite{nickl2012confidence}.}  Our results imply all the upper bounds of~\cite{nickl2012confidence} in all cases (i), (ii) and (iii) of Table~\ref{fig:1}, i.e.~imply the upper bounds on exactly sparse sets (this is illustrated in Theorem~\ref{cor:conf}). 
Also, our confidence set is adaptive and honest in all three cases (i), (ii) and (iii), and we do not need to change the construction method as in paper~\cite{nickl2012confidence}. Our assumptions on the design of $X$ are weaker than in the paper~\cite{nickl2012confidence}, where the authors consider Gaussian design which a fortiori satisfies Assumption~\ref{ass:design} with high probability. Finally, the confidence set is, as we saw, computationally feasible, since its computational complexity is of order $np$. As mentioned in the introduction, the procedure in the paper~\cite{nickl2012confidence} boils down to minimizing over the set of $S_0-$ sparse vectors a quadratic quantity, which has complexity of order $p^{S_0}n$. This implies that our procedure is computationally efficient on a set that is as large as possible in a minimax sense, as illustrated by the lower bounds in Figure~\ref{fig:1}.


\subsection{Adaptive and honest confidence sets for multiple sparsities}\label{ss:estS}

\paragraph{Construction of the confidence set}

In the last subsection, we restricted ourselves to constructing a confidence set that is adaptive to only two sparsities $S_0, S_1$. Although it is already useful with respect to existing techniques that are not adaptive at all, it is only a first step toward a more global result where all sparsity indexes $\mathcal I = \{1, \ldots, \bar p\}$ are considered. Algorithm~\ref{alg:21} solves this problem.



\vspace{-0.3cm}

\begin{algorithm}[h!]
\caption{Multi sparsity indexes confidence set, second version}
 \label{alg:21}
 \begin{algorithmic}
\STATE \textbf{Parameters:} $\delta, \sigma^2$
\SET using only the first half of the $2n$ samples,
$$\hat B^2:= 3/2 \Big(n^{-1}\sum_{i \leq n} Y_i^2 (1+2\log(1/\delta)) + 2\log(1/\delta)\Big),$$
and $\tau_n(S) := 14|\hat B|\sqrt{n^{-1/2}\log(1/\delta)}+ 381|\hat B|\sqrt{\frac{S\log(p/\delta)}{n}}$,
and $\tau_n'(S) := 330|\hat B| \sqrt{\frac{(S+1)\log(p/\delta)}{n}}$,
and let $\hat \theta$ be the lasso estimate as in Theorem~\ref{th:adaest}.

\SET the residual $\hat r = Y - X \hat \theta$
\SET the statistic $R_n = \|\hat r\|_2^2 -  \sigma^2$
\SET for any $S \leq p$ the statistic  $R_n'(S):=  \sum_{j=S+1}^p \hat \theta_{(j)}^2$,
where $\hat \theta_{(.)}$ is the ordered version of $|\hat \theta|$.

\SET $\hat S$ as the smallest $S\leq p$ such that $R_n \leq \tau_n(S)^2, \quad \mathrm{and}, \quad R_n'(S) \leq (\tau_n'(S))^2$.

\SET the confidence interval 
$$C_n := \big\{u \in \mathbb R^p: \|u - \hat \theta\|_2 \leq 650 \sqrt{\frac{\hat S \log(p/\delta)}{n}}.$$
\STATE \textbf{return} $C_n$

\end{algorithmic}
\end{algorithm}
The parameter $\sigma^2$ car be replaced by a consistent estimator of the variance of the noise $\epsilon$ (e.g.~a Bootstrap estimate, a cross validation estimate, etc).

The following theorem holds in this case (it is a direct consequence of Theorem~\ref{cor:conf}).
\begin{theorem}\label{cor:bigalter2}
Assume that the noise is either Gaussian of variance less than $1$, or bounded by $1$, and assume that the assumptions for convergence of the adaptive lasso, stated in Theorem~\ref{th:adaest}, are satisfied, and that for $S_1 \leq \bar p$. Then the confidence set presented in Algorithm~\ref{alg:21} is $\delta$ adaptive and honest for $\mathcal I = \{1, \ldots, \bar p\}$ and over the model
$$\mathcal P:=\mathcal S_{1} (32,\infty, \bar p, \delta) \bigcup \bigcup_{S=2}^{\bar p} \tilde {\mathcal S}_{S, S-1} (32,\infty,\bar p, \delta,\rho_n(S)).$$
where
\small{
\begin{align*}
&\tilde \rho_n(S) = |\hat B| \min\Big(50\sqrt{\log(1/\delta)} n^{-1/4}, 460\sqrt{\frac{(S+1) \log(p/\delta)}{n}}\Big).
\end{align*}}
\end{theorem} 
A more general procedure is presented in the Supplementary Material, Appendix~\ref{app:expl}.

\paragraph{Discussion} This result is minimax optimal from the lower bounds in Figure~\ref{fig:1}, and it is also computationally feasible. The resulting confidence interval is adaptive and honest for all indexes $\mathcal I$ over $\mathcal P$. Moreover, $\mathcal P$ is significantly larger than the set of ``detectable" parameters such that all non-zero component are larger than $\sqrt{\log(p/\delta)/n}$. For this reason, this method is more efficient than the naive method of counting the number of non-zero entries in a thresholded adaptive estimate, and using this number for constructing the confidence set. 

\section{Experimental results}

In this section, we present some simulations and also some applications on images.

\subsection{Simulations}

In order to illustrate the efficiency of our method, we apply it to simulated data. We consider a problem in dimension $p=10000$, and where $n=1000$ (the sampling rate is $10\%$). Let $0<S_0<S_1$ be the two approximate sparsity levels. We are going to define three types of distributions (priors) on the set of parameter $\theta$:
\begin{itemize}
\item $\theta \sim \Theta_1$: (i) $S_0$ random coordinates of $\theta$ are $\mathcal N(0,1)$ and (ii) the remaining coordinates are $\mathcal N(0,\sigma_0^2)$ where $\sigma_0^2= \frac{S_0\log(p)}{np}$. With high probability, $\theta \in \mathcal S_{S_0}(C,B, \bar p, \delta)$.
\item $\theta \sim \Theta_2$: (i) $S_1$ random coordinates of $\theta$ are $\mathcal N(0,1)$ and (ii) the remaining coordinates are $\mathcal N(0,\sigma_0^2)$. With high probability, $\theta \in  \tilde{\mathcal S}_{S_1}(C,B, \bar p, \delta, \rho_n^2)$ where $\rho_n^2 = O(\frac{S_1\log(p)}{n} + \frac{S_0\log(p)}{n})$.
\item $\theta \sim \Theta_3$: (i) $S_0$ random coordinates of $\theta$ are $\mathcal N(0,1)$ and (ii) the remaining coordinates are $\mathcal N(0,\sigma_1^2)$ with $\sigma_1^2= C\left(\frac{1}{n^{1/2}p} + \frac{S_0\log(p)}{np}\right)$. With high probability, $\theta \in \tilde{\mathcal S_{S_1}}(C,B, \bar p, \delta, \rho_n^2)$, where $\rho_n^2 = O(n^{-1/2} + \frac{S_0\log(p)}{n})$.
\end{itemize}

See Figure~\ref{fig:exp2} for an illustration of this.

\begin{center}
\begin{figure}[h]
\begin{center}
\includegraphics[width=7cm]{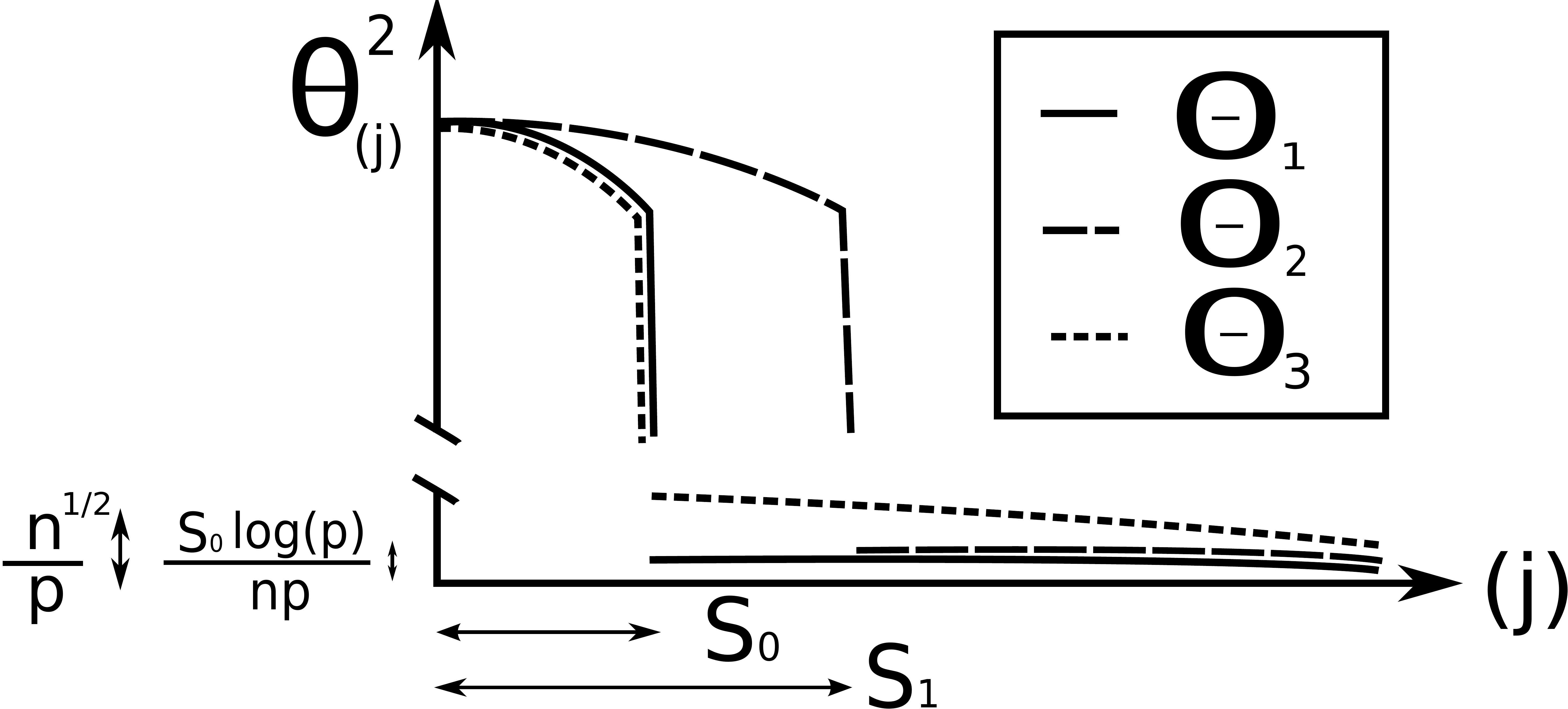}
\caption{Mean of the square of the re-ordered coordinates $\theta_{(j)}^2$ of $\theta$ sampled according to priors $\Theta_1$, $\Theta_2$, and $\Theta_3$.} \label{fig:exp2}
\end{center}
\end{figure}
\end{center}

The distributions correspond to two extremal cases in which the sampled vector $\theta$ is not approximately sparse, i.e.~either the norm of the tail coefficients is large, or the number of detectable coefficients is larger than $S_0$. For a given $\theta \sim \Theta_k$ ($k \in \{1,2,3\}$), we write $\Psi$ for its ``class", i.e.~if the hypothesis in test~\eqref{test:2} to which they belong with high probability. This means that $\Psi=0$ for $\theta \sim \Theta_1$ and $\Psi=1$ for $\theta \sim \Theta_2$ or $\theta \sim \Theta_3$.

For all distributions $\Theta_k$, we do $10000$ experiments (corresponding to trying our methods over $10000$ samples of $\theta$ sampled according to $\Theta_k$) where we perform the test we described in Section~\ref{sec:test} for the testing problem~\eqref{test:2}, infer the sparsity, compute adaptive confidence interval, and compute the risk of the estimate. The sampling matrix $X$ is composed of Gaussian random variables, and the noise $\epsilon$ is i.i.d.~Gaussian with variance $1$. In this design, $\bar p = O(n/\log(p))$. We compute an estimate of $\theta$ via hard thresholding, which happens in this orthogonal setting to be equivalent to lasso, on the first half of the samples. We then construct the test on the second half of the sample. We summarise the results in Table~\ref{f:sim2}. 

\begin{table*}[t]
\begin{footnotesize}
\begin{center}
\begin{tabular}{|c||ccc|ccc|ccc|}
  \hline
 & $S_1 = 5$ &  and &  $S_2 = 10$ & $S_1 = 5$ &  and &  $S_2 = 10^3$ & $S_1 = 50$ &  and &  $S_2 = 10^3$ \\
\textbf{Prior} & $\Theta_1$ & $\Theta_2$ & $\Theta_3$& $\Theta_1$ & $\Theta_2$ & $\Theta_3$& $\Theta_1$ & $\Theta_2$ & $\Theta_3$\\
\hline
\hline
  $\mathbb P_{\theta}(\Psi_n \neq \Psi)$ & $5\hspace{1mm}10^{-2}$ & $1\hspace{1mm}10^{-2}$ & X & $1\hspace{1mm}10^{-2}$ & $8\hspace{1mm}10^{-7}$  &   $8\hspace{1mm}10^{-2}$ & $1\hspace{1mm}10^{-3}$&   $2\hspace{1mm}10^{-5}$ &  $6\hspace{1mm}10^{-3}$\\
  \hline
 $\mathbb E \|\hat \theta\|_0$ &   $4.8$ & $9.7$ & X & $4.8$ &  $2.1\hspace{1mm}10^{3}$ & $4.8$ &$48.7$&  $2.3\hspace{1mm}10^{3}$ &  $49.3$\\
  \hline
  $\mathbb P_{\theta}(\theta \not\in C_n)$ &  $1.2\hspace{1mm}10^{-1}$&  $9\hspace{1mm}10^{-2}$ & X & $5\hspace{1mm}10^{-2}$ &  $9.10^{-7}$ &   $8.1 \hspace{1mm}10^{-2}$ & $3.2\hspace{1mm}10^{-3}$&  $4\hspace{1mm}10^{-5}$ &  $7.8\hspace{1mm}10^{-3}$\\
  \hline
 $\mathbb E |C_n|$ &  $6.16\hspace{1mm}10^{-2}$&  $1.8\hspace{1mm}10^{-1}$ & X & $7.76\hspace{1mm}10^{-2}$ & $9.8 \hspace{1mm}10^{3}$  & $5.7\hspace{1mm}10^{-1}$ & $27.6$&  $9.9\hspace{1mm}10^{3}$&  $9.7\hspace{1mm}10^{3}$\\
  \hline
 $\mathbb E \|\hat \theta - \theta\|_2^2$ &  $3.37\hspace{1mm}10^{-2}$& $1.1\hspace{1mm}10^{-1}$ & X & $3.7 \hspace{1mm}10^{-2}$ &  $9.6\hspace{1mm}10^{3}$ & $3.2\hspace{1mm}10^{-1}$ & $15.04$&  $9.7\hspace{1mm}10^{3}$& $9.4\hspace{1mm}10^{3}$\\
  \hline
\end{tabular}
\end{center}
\caption{Expected results for the test, risk and adaptive and honest confidence set for the three priors.}\label{f:sim2}
\end{footnotesize}
\end{table*}

A first general remark is that the test we consider manages to distinguish efficiently between $H_0$ and $H_1$, for many different configuration of sparsity. The adaptive and honest confidence sets we built using this test are also quite efficient. The probability that the true parameter belongs to the adaptive confidence set is very close to the probability of correctly inferring the class of $\theta$. The strength of these sets is to be adaptive to the sparsity of the problem, i.e.~they contain $\theta$ with high probability and do not have the same width depending on the complexity of the true parameter ($S_0$ or $S_1$). As a matter of fact, the width of the adaptive confidence set is close to the value of the risk of the adaptive estimate, which is exactly what is wanted. It is particularly interesting, since as expected, the risk is much larger under $H_1$ than under $H_0$.

In the case of distribution $\Theta_3$, it is interesting to remark that although the sparsity of $\hat \theta$ is close to $S_0$ in expectation, it does not prevent our test to efficiently classify it as $H_1$. This is actually quite important in terms of confidence sets since we can observe that, for each configuration of sparsity, the risk and thus the width of the adaptive confidence interval, is much larger for $\Theta_3$ than for $\Theta_1$. A test only based on the inferred sparsity (i.e.~on $\|\hat \theta\|_0$) would not have detected this since the inferred sparsity is approximately similar in these two cases.

\subsection{Application of the method to images}

\begin{center}
\begin{figure}[t]

\begin{minipage}{1.7cm}
\includegraphics[width=2.7cm]{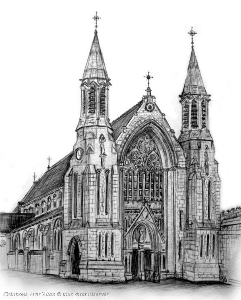}
\end{minipage}
\hspace{1cm}
\begin{minipage}{1.7cm}
\includegraphics[width=2.7cm]{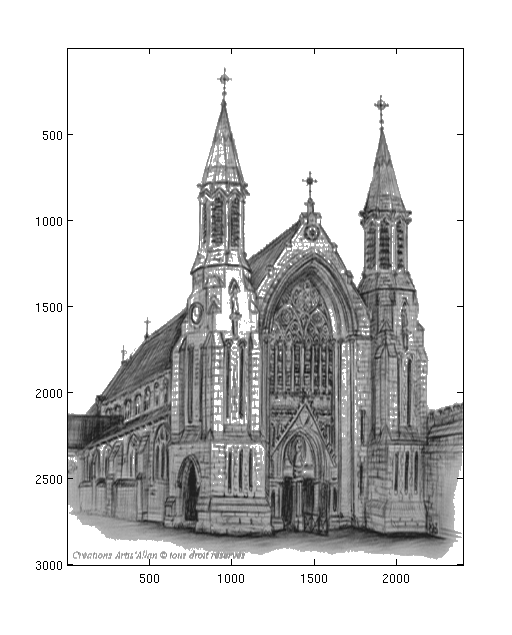}
\end{minipage}
\hspace{1cm}
\begin{minipage}{1.7cm}
\includegraphics[width=2.7cm]{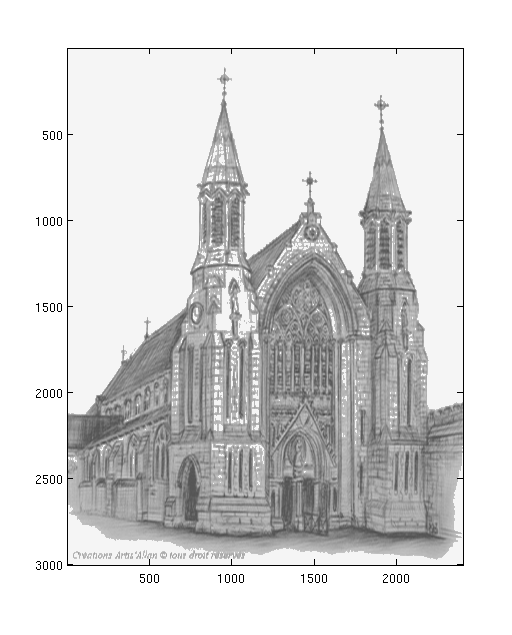}
\end{minipage}

\begin{minipage}{1.7cm}
\includegraphics[width=2.7cm]{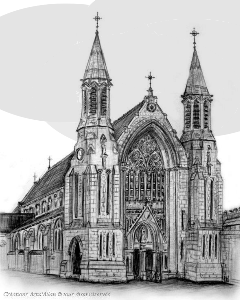}
\end{minipage}
\hspace{1cm}
\begin{minipage}{1.7cm}
\includegraphics[width=2.7cm]{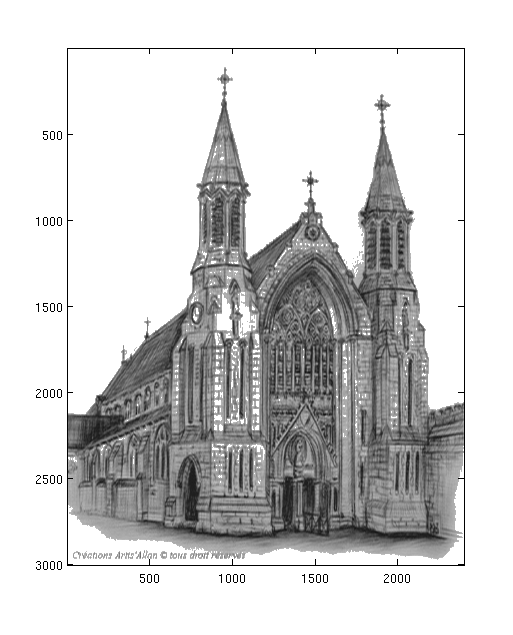}
\end{minipage}
\hspace{1cm}
\begin{minipage}{1.7cm}
\includegraphics[width=2.7cm]{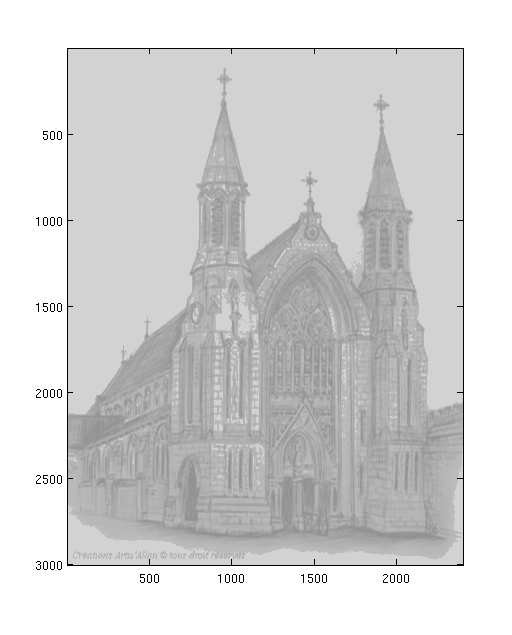}
\end{minipage}

\caption{First column = original image. Second column= reconstructed image. Third column: extremal point of the confidence set that minimises the contrast. The test of $3\% p$ sparsity is accepted for the first image but rejected for the other.} \label{fig:exp3}
\end{figure}
\end{center}

We consider now a more concrete setting, where we apply our method to images. We focus on black and white drawings\footnote{Note that a pre-treatment, like differentiation, can be applied to regular images to transform them into drawing-like images.}. The particularity of such images is that only a few pixels are non-zero, i.e.~if we align the columns of such an image, it can be considered as a sparse vector $\theta$ of dimension $p$ where $p$ is the number of lines times the number of columns of the image. Such an image $\theta$ can easily be compressed by conserving $n$ Fourier coefficient of this vector, chosen at random frequencies. We write $X$ the $n \times p$ matrix that represents this convolution. The model is then again $Y=X\theta + \epsilon$, where $\epsilon$ is some noise to the compression, due for instance to transmission. An interesting question when observing the compression $Y$ of an image is to infer what is the quality of the reconstruction, i.e.~if the image is indeed quite sparse or not. 

When one then observes such a compression, it is possible to reconstruct the image $\theta$ by, as we explained in Section~\ref{sec:est}. 
Also, the test and the confidence sets are build in the same way as what was done in Section~\ref{sec:test}.

We consider here an image with $p=7200000$ pixels. We consider $n=5\% p = 36000$ Fourier coefficients obtained by FFT (Fast Fourier Transform). We consider $2$ images that we display in Figure~\ref{fig:exp3} (first column), and write $\theta^{(k)}$ where $k \in \{1,2\}$. The first image is a black and white drawing of a cathedral, and the second one is the same drawing but with a background (a cloud): the first one will be approximately sparse and the second one not. It is possible to compress these images by considering, instead of $\theta^{(k)}$, the vector $Y^{(k)}$. The quality of the compression, i.e.~the proximity between the image reconstructed trough $Y^{(k)}$ and the image $\theta^{(k)}$, will depend however very much on the sparsity of the image, as we saw in Theorem~\ref{th:adaest}. We can use the results of Section~\ref{sec:test} to test whether the image is a least $3\%p$ approximately sparse or not, and then build confidence sets around it. Figure~\ref{fig:exp3} (two last columns, the first one containing the reconstructed images and the second one an extremal point of the confidence sets) illustrates this. More precisely, we display, for each image, the estimate of $\theta^{(k)}$ (i.e.~the reconstructed image), and an extreme points of the confidence sets that we choose as being the image that minimises the contrast.

Although image $1$ and $2$ are different images, their estimates are very close. The test reveals the fact that they are not the same, and that in particular the reconstruction of image $1$ will be good while the reconstruction of image $2$ will be bad (although they seem similar from their reconstruction). The confidence sets also show how much the true image could actually be different from the reconstructed image. In particular, the extremal point of the confidence set that minimises the contrast implies that although it is rather unlikely that there is a background in image $1$, image $2$ might well have one. For these images, the notion of approximate sparsity is very important since even the first image is not at all sparse (not even $40\% p$ sparse). It is however less than $3\% p$ approximately sparse. Because of the background, however, the second image is not even close to $3\% p$ approximately sparse.

\vspace{-0.3cm}

\paragraph{Conclusion} In this paper, we developed a computationally feasible, adaptive and honest confidence interval, first in the two sparsity indexes case, and then in the general setting of multi sparsity indexes. The method we propose is efficient on a maximal set (in a minimax-optimal sense), and is implementable, which is a novelty with respect to the existing results. The assumptions we make are also less restrictive than what was previously required.
 We also provided an experimental validation of this results by simulations, and also an application on images.








\newpage

\appendix
\onecolumn

{\huge Supplementary material for the paper: ``Implementable confidence sets in high dimensional regression''}

\section{General construction of the confidence sets}\label{app:expl}

\subsection{General construction of a confidence set in the two point case}

\paragraph{Construction of the confidence set}

Let $C,B,c, \sigma^2,\bar p>0$, and $\delta>0$. Let $S_0<S_1 <\bar p$. The algorithm for the two sparsity indexes, Algorithm~\ref{alg:1}, contains two main steps. The first step is to construct a test $\Psi_n$ for deciding whether $\theta$ is $S_0$ approximately sparse, or $S_1$ approximately sparse. The test consists in first computing an adaptive estimate $\hat \theta$ of $\theta$, and then on thresholding all non-significant components. Then, the testing decision is based on two factors (i) testing whether the number of non-zero entries of the thresholded estimate is larger than $S_0$ and (ii) testing whether the squared residuals $\|\hat r\|_2^2$ are significant or not. If both these quantities are small enough, the test is accepted, otherwise it is rejected. The outcome of this is the test $\Psi_n$. The second step is to use this information to construct the confidence set $C_n$. Based on $\Psi_n$, the confidence interval $C_n$ will be of diameter of order $\sqrt{\frac{S_0\log(p)}{n}}$ (if $\Psi_n = 0$), or $\sqrt{\frac{S_1\log(p)}{n}}$ (if $\Psi_n = 1$). The procedure is explained in Algorithm~\ref{alg:1}.


\begin{algorithm}[h]
\caption{Two sparsity indexes confidence set}
 \label{alg:1}
 \begin{algorithmic}
\STATE \textbf{Parameters:} $C,B,\bar p, \delta, \sigma^2, c_m, C_M$
\SET $\hat \theta = $ an adaptive estimate of $\theta$ as in Theorem~\ref{th:adaest} or~\ref{th:adaest2}, computed on the first half of the sample only
\STATE Consider from now on only the second half of the sample ($X,Y:=$ second half of the sample)
\SET the constants $\tau_n = 3\sqrt{\frac{C}{\min(c_m,1)} n^{-1/2}\log(1/\delta)}+ 2\sqrt{\frac{(C_M+1)}{\min(c_m,1)}E\frac{S_0\log(p/\delta)}{n}}$, and $\tau_n' = 2 \sqrt{\frac{E}{\min(c_m,1) }\frac{S_1\log(p/\delta)}{n}}$, where $E:=\Big(12(36\kappa + 36)^2 + C^2 \Big)$ if $\hat \theta$ is the estimate from Theorem~\ref{th:adaest}, otherwise $E$ as in Theorem~\ref{th:adaest2}
\SET the residual $$\hat r = Y - X \hat \theta$$
\SET the test statistic $$R_n = \|\hat r\|_2^2 - \sigma^2$$

\SET the test
\begin{equation}\label{eq:testS1gd}
\Psi_n = 1 -  \mathbf 1\{R_n \leq \tau_n^2\} \mathbf 1 \{  \sum_{j=S_0+1}^p \hat \theta_{(j)}^2 \leq (\tau_n')^2\},
\end{equation}
where $\hat \theta_{(.)}$ is the ordered version of $|\hat \theta|$

\SET the confidence interval
\begin{align*}
C_n &:= \big\{u \in \mathcal S_{S_1}: \|u - \hat \theta\|_2 \leq  \sqrt{E \frac{S_0 \log(p)}{n}} \mathbf 1\{\Psi_n = 0\}\\
&+ \sqrt{E \frac{S_1 \log(p)}{n}} \mathbf 1\{\Psi_n = 1\}\big\}
\end{align*}
\STATE \textbf{return} $C_n$
\end{algorithmic}
\end{algorithm}

\paragraph{Main result}

The following theorem states that this confidence interval is adaptive and honest (in the sense of Definition~\eqref{def:csetint}) for $\rho_n$ large enough.
\begin{theorem}\label{th:bigalter}
Let Assumptions~\ref{ass:noise} and~\ref{ass:design} be verified for $\bar p>S_1>S_0>0$, and for constants $c, c_m, C_M>0$. Let $\delta>0$ and $B>0$ and $C \geq 16C' (B^2 + c^2)$ where $C'$ is some universal constant. Assume that 
\begin{align*}
\rho_n &\geq \frac{3(C_M+1)}{\sqrt{\min(c_m,1)}}\\
&\times\min\Big(\sqrt{C\log(1/\delta)} n^{-1/4}, \sqrt{E}\sqrt{\frac{S_1 \log(p/\delta)}{n}}\Big).
\end{align*}
if $S_0 \leq n^{1/2}\log(1/\delta)/\log(p/\delta)$, and $\rho_n = 0$ otherwise.
Set
$$\mathcal P:=\mathcal P_n(\rho_n) = \mathcal S_{S_0} (C,B, \bar p, \delta) \bigcup \tilde {\mathcal S}_{S_1} (C,B,\bar p, \delta,\rho_n).$$
The confidence set $C_n$ constructed in Algorithm~\ref{alg:1} is $\delta-$adaptive and honest for $\mathcal I = \{S_0 ,S_1\}$ and $\mathcal P$.
\end{theorem}
The proof of this theorem is in the supplementary material (Appendix~\ref{app:2}).

\paragraph{Discussion}

\textit{Comparison with results in paper~\cite{nickl2012confidence}.} The results we obtain hold uniformly on any vector of $S_{S_0}(C,B, \bar p, \delta)$, and $\tilde {\mathcal S}_{S_1} (C,B, \bar p, \delta, \rho_n)$ for $\rho_n$ as in Theorem~\ref{cor:conf}. Let $l_2(B)$ be the $l_2$ ball of radius $B$. It holds by definition of these sets that
\begin{align*}
l_0(S_0) \cap l_2(B) &\subset \mathcal S_{S_0}(C,B, \bar p,\delta),\\
&\mathrm{and} \\
\tilde {l_0}(S_1, (1+C)\rho_n) \cap &l_2(B) \subset \tilde {\mathcal S}_{S_1} (C,B, \bar p, \delta, \rho_n).
\end{align*}
In particular, our results imply all the upper bounds of~\cite{nickl2012confidence} in all cases (i), (ii) and (iii) of Table~\ref{fig:1}, i.e.~imply the upper bounds on exactly sparse sets (this is illustrated in Theorem~\ref{cor:conf}). 
Also, our confidence set is adaptive and honest in all three cases (i), (ii) and (iii), and we do not need to change the construction method as in paper~\cite{nickl2012confidence}. Our assumptions on the design of $X$ are weaker than in the paper~\cite{nickl2012confidence}, where the authors consider Gaussian design which a fortiori satisfies Assumption~\ref{ass:design} with high probability. Finally, the confidence set is, as we saw, computationally feasible, since its computational complexity is of order $np$. As mentioned in the introduction, the procedure in the paper~\cite{nickl2012confidence} boils down to minimizing over the set of $S_0-$ sparse vectors a quadratic quantity, which has complexity of order $p^{S_0}n$. This implies that our procedure is computationally efficient on a set that is as large as possible in a minimax sense, as illustrated by the lower bounds in Figure~\ref{fig:1}.

\textit{The parameters of the algorithm.} The definition of the test $\Psi_n$ uses the knowledge of the RIP constants of the matrix $X$, the input of the parameters $C,B,\bar p, \delta$ of the enlarged sets, and the knowledge of the variance $\sigma^2$ of the noise. The dependency on $\sigma^2$ is undesirable, but there is a way around it by dividing the second half of the dataset in $2$ parts, and using the first part to estimate $\sigma^2$, and the second part to construct the confidence set (plugging in the estimate of $\sigma^2$ instead of the true value). This is explained in Corollary~\ref{cor:conf}. The parameters $B,C, \bar p, \delta$ of the enlarged sets are rather given to choose to the user for flexibility. The parameter $\delta$ is the desired coverage of the confidence set, the parameters $C, \bar p$ calibrate the desired level of approximated sparsity, the parameter $B$ is an upper bound on the $l_2$ norm of $\theta$. It is reasonable that the user fixes $\delta$, as well as the parameters $C, \bar p$. For instance, taking any fixed constant $C\geq 0$ and setting $\bar p =0$ will already allow to deal with all parameters in $l_0(S)$, see Corollary~\ref{cor:conf} where we choose $C$ large enough ($C=32$) so that the assumptions of Theorem~\ref{th:bigalter} are satisfied, and where we set $\bar p = 0$ in the second part of the corollary. The parameter $B^2$, i.e.~an upper bound on $\|\theta\|_2^2$, can be estimated with high enough precision by $(\min(1,c_m))^{-1}\Big(n^{-1}\sum_{i \leq n} Y_i^2 (1+2\log(1/\delta)) + 2\log(1/\delta)\Big)$, if the noise is either bounded by $1$, or is a Gaussian of variance $1$. The RIP constants of the matrix $X$, namely $c_m$ and $C_M$ are not computable in finite time. However, it is not necessary for the good functioning of the algorithm to have an exact knowledge of these constants. An upper bound on $C_M$ and a lower bound on $c_m$ suffice, and this will only damage the performance of the algorithm by a constant factor with respect to what could be achievable if $c_m, C_M$ were known precisely. Moreover, if $c_m$ is to small, and $C_M$ is too large, it means that the design is very correlated, and in this case, all known computable estimates fail. The lasso, for instance, works only for a quite restricted range of values for the parameters $c_m$ and $C_M$ (in this paper, we reprove the lasso consistency in the enlarged sets, and obtain that it functions correctly for $c_m \geq 2/3$ and $C_M \leq 4/3$, see Theorem~\ref{th:adaest}, this can be slightly improved, but not by much). Taking that into account, it is hopeless to hope for a confidence set that will be at the same time computable and adaptive and honest for all values of $c_m, C_M$: indeed, as pointed up above, any point of an adaptive and honest confidence set is actually a minimax adaptive estimator. In Corollary~\ref{cor:conf}, we precise all values of the parameters for the construction of the confidence set, under the assumption that the lasso estimate is efficient (under the assumptions of Theorem~\ref{th:adaest}): in this case the values of the parameters are fully explicit and the confidence set can be computed without a priori knowledge on the problem.

\subsection{General construction for multiple sparsities}

\paragraph{Construction of the confidence set}

In the last subsection, we restricted ourselves to constructing a confidence set that is adaptive to only two sparsities $S_0, S_1$. Although it is already useful with respect to existing techniques that are not adaptive at all, it is only a first step toward a more global result where many sparsity indexes $\mathcal I$ are considered (potentially, $\mathcal I$ can be the whole set of indexes $\{1, \ldots, \bar p\}$).

Let $\mathcal I = \{i_1, \ldots, \i_I\}$, where $i_1 \leq i_2 \leq \ldots \leq i_I$ be the set of sparsities over which one wishes to adapt. An intuitive extension of the two sparsities index confidence set to this more complex setting is presented in Algorithm~\ref{alg:2}. The idea is to compute the tests $\Psi_n$ for the adjacent sparsity indexes in $\mathcal I$, for smaller and smaller sparsities, until the sparsity index $t$ where the test is rejected ($i_{t-1}$ is rejected). The confidence band considered is then chosen of diameter of order $\sqrt{i_t\log(p)/n}$, and is centred in an adaptive estimate $\hat \theta$.

\begin{algorithm}[h!]
\caption{Multi sparsity indexes confidence set}
 \label{alg:2}
 \begin{algorithmic}
\STATE Set $\Psi_n := 0$, $t=I$
\WHILE{$\Psi_n =0$}
\STATE Set $\Psi_n$ and $C_n$ as the test and confidence interval outputted by Algorithm~\ref{alg:1} with parameters $S_0:=i_{t-1}, S_1:=i_t$
\ENDWHILE
\STATE \textbf{return} $C_n$
\end{algorithmic}
\end{algorithm}

\vspace{-0.3cm}

\paragraph{Main result}

This confidence set will be adaptive and honest for a model that respects the restriction of all sparsity indexes in $\mathcal I$.

\begin{theorem}\label{cor:bigalter}
Let Assumptions~\ref{ass:noise} and~\ref{ass:design} be satisfied for $\bar p>0$, and for constants $c, c_m,C_M>0$. Let $\delta>0$ and $B>0$ and $C \geq 16C' (B^2 + c^2)$ where $C'$ is some universal constant. Let, for any $2 \leq t \leq I$
\begin{align*}
\rho_n(t) &\geq \frac{3(C_M+1)}{\sqrt{\min(c_m,1)}}\\
&\times\min\Big(\sqrt{C\log(1/\delta)} n^{-1/4}, \sqrt{E}\sqrt{\frac{i_{t} \log(p/\delta)}{n}}\Big),
\end{align*}
if $i_{t-1} \leq n^{1/2}\log(1/\delta)/\log(p/\delta)$, and $\rho_n(t) = 0$ otherwise. Set
$$\mathcal P:=\mathcal S_{i_1} (C,B, \bar p, \delta) \bigcup \bigcup_{t=2}^I \tilde {\mathcal S}_{i_t, i_{t-1}} (C,B,\bar p, \delta,\rho_n(t)).$$
The confidence set $C_n$ presented in Algorithm~\ref{alg:2} is $\delta-$adaptive and honest for $\mathcal I$ and $\mathcal P$.
\end{theorem}
This theorem follows immediately from Theorem~\ref{th:bigalter}, and an union bound over all tests.





\section{Proof for Theorem~\ref{th:adaest}}\label{app:12}


\paragraph{Step 0: Compatibility condition}

We state the compatibility condition for $X$, or restricted eigenvalue condition (see~\cite{bickel2009simultaneous}).

\begin{assumption}\label{ass:compcond}
Let $\phi>0$ and $\bar p>0$. Let $U$ be any support of size $S \leq \bar p$ of $\{1...p\}$, $u$ be any vector, and $u_U$ be the restriction of $u$ to $U$. The $(\bar p, \phi)-$compatibility condition is satisfied if for any such $U, u$ such that $\|u_{U^C}\|_1 \leq 4 \|u_{U}\|_1$ where $U^C$ is the complement of $U$, it holds that
\begin{align*}
\|u_{U}\|_1^2 \leq \frac{S\|X u\|_2^2}{n\phi^2}.
\end{align*}
\end{assumption}

Note that this assumption is implied by the R.I.P. property with not too large R.I.P. constant for $33\bar p$ vectors, i.e.~$c_m>2/3$ and $C_M <4/3$ in Assumption~\ref{ass:design} with $33\bar p$. To deduce this, combine Assumption 2 in~\cite{bickel2009simultaneous}, with $m = 32 \bar p$ and $c_m>2/3$ and $C_M <4/3$, with the definition of $RE(\bar p,m,4)$ and Lemma 4.1 in~\cite{bickel2009simultaneous}, to obtain that this RIP property implies this compatibility condition with $\phi^2 = 1/6$. See also~\cite{zou2006adaptive}.

Since in the assumptions of Theorem~\ref{app:12}, we assume Assumption~\ref{ass:design} with $66\bar p$, $c_m>2/3$ and $C_M <4/3$, we know from the previous remark that Assumption~\ref{ass:compcond} holds with $2\bar p$ and $\phi^2 = 1/6$ .

\paragraph{Step 1: Decomposition of the problem}

Let $\theta \in \mathcal S_S(C,B,\bar p,\delta)$. Let $\theta = \theta_1 + \theta_2$, where $\theta_1$ contains the $S-$largest components of $\theta$, and $\theta_2$ the rest. We have for any vector $u \in \mathbb R^p$
\begin{align}
\|Y - Xu\|_2^2  &= \|\epsilon + X \theta_2\|_2^2 + 2\langle \epsilon + X\theta_2, X(\theta_1 - u) \rangle +\|X (\theta_1 - u)\|_2^2 \nonumber\\ 
&= \upsilon^2  +\|X (\theta_1 - u)\|_2^2 + 2\langle \epsilon, X(\theta_1 - u) \rangle + 2\langle X\theta_2, X(\theta_1 - u) \rangle, \label{eq:acici1}
\end{align}
where $\upsilon^2 = \|\epsilon + X \theta_2\|_2^2$ which does not depend on $u$.

Note first that
\begin{align}
|\langle  X\theta_2, X(\theta_1 - u) \rangle| &\leq \|X\theta_2\|_2 \|X(\theta_1 - u) \|_2 \nonumber\\
&\leq C_M\sqrt{n}\|\theta_2\|_2 \|X(\theta_1 - u) \|_2 \nonumber\\
&\leq  C_M\sqrt{C S\log(p/\delta)}\|X(\theta_1 - u) \|_2, \label{eq:acici2}
\end{align}
by definition of $X$, since $\theta  \in \mathcal S_S(C,B,\bar p,\delta)$ and is thus $\bar p$ sparse, and since $\theta_2$ contains all entries of $\theta \in \mathcal S_S(C,B,\bar p,\delta)$ that are smaller than the $S$ largest (and thus, $\|\theta_2\|_2 \leq \sqrt{\frac{C S\log(p/\delta)}{n}}$).

We have since $\epsilon$ is a centred sub-Gaussian process, and by H\"older's inequality and an union bound, that with probability larger than $1-\delta$ 
\begin{align}\label{eq:acici3}
|\langle  \epsilon, X(\theta_1 - u) \rangle| \leq \|X^T \epsilon\|_{\infty} \|(\theta_1 - u) \|_1 \leq Dc \sqrt{\log(p/\delta)n} \|\theta_1 - u \|_1,
\end{align}
where $D$ is some universal constant. It holds since for any $j \leq p$, with probability larger than $1-\delta$ , $\langle X_{.,j}, \epsilon \rangle \leq Dc \sqrt{\log(1/\delta)n}$ by H\"older's inequality (since $\epsilon$ sub-Gausssian). Then an union bound on all $j\leq p$ leads to the bound~\eqref{eq:acici3}.

 Equations~\eqref{eq:acici1} to~\eqref{eq:acici3} imply that with probability larger than $1-\delta$ 
\begin{align}
\|Y - Xu\|_2^2 &\geq  \upsilon^2 + \|X (\theta_1 - u)\|_2^2 - Dc \sqrt{\log(p/\delta)n} \|\theta_1 - u \|_1 -  C_M\sqrt{C S\log(p/\delta)}\|X(\theta_1 - u) \|_2.  \label{eq:acici6}
\end{align}

\paragraph{Step 2: Lasso estimator}

Define now the estimator $\hat \theta$ as being
\begin{align*}
\hat \theta &= \arg\min_u \Big[ \|Y - Xu\|_2^2 +  \kappa \sqrt{\log(p/\delta)n}\|u\|_1\Big],
\end{align*}
where $\kappa> \max(2A, \frac{4\phi}{3}A^2)$ where $A:=\max(Dc, C_M\sqrt{C})$.

By definition of $\hat \theta$, we have if we compare it to the point $u:=\theta_1$ that
\begin{align*}
 \|Y - X\hat \theta\|_2^2 +  \kappa \sqrt{\log(p/\delta)n} \|\hat \theta\|_1 \leq  \|Y - X \theta_1\|_2^2 +  \kappa \sqrt{\log(p/\delta)n} \|\theta_1\|_1.
\end{align*}

From Equation~\eqref{eq:acici6} and Equation~\eqref{eq:acici1}, this implies
\begin{align*}
&\upsilon^2 + \|X (\theta_1 - \hat \theta)\|_2^2 - Dc \sqrt{\log(p/\delta)n} \|\theta_1 - \hat \theta \|_1 - C_M\sqrt{CS\log(p/\delta)}\|X(\theta_1 - \hat \theta) \|_2  +  \kappa \sqrt{\log(p/\delta)n} \|\hat \theta\|_1\\
&\leq   \upsilon^2  +  \kappa \sqrt{\log(p/\delta)n} \|\theta_1\|_1,
\end{align*}
which is equivalent to
\begin{align*}
&\|X (\theta_1 - \hat \theta)\|_2^2   +  \kappa \sqrt{\log(p/\delta)n} \|\hat \theta\|_1\\
&\leq   Dc \sqrt{\log(p/\delta)n} \|\theta_1 - \hat \theta \|_1 + C_M\sqrt{CS\log(p/\delta)}\|X(\theta_1 - \hat \theta) \|_2 +  \kappa \sqrt{\log(p/\delta)n} \|\theta_1\|_1.
\end{align*}
This implies since $A = \max(Dc, C_M\sqrt{C }) $, since $\theta_1$ is $S-$sparse, that with probability $1-\delta$
\begin{align*}
 &\|X (\theta_1 - \hat \theta)\|_2^2  + \kappa \sqrt{\log(p/\delta)n} (\|\hat \theta_{\Theta^C}\|_1 + \|\hat \theta_{\Theta}\|_1)\\ 
&\leq \kappa \sqrt{\log(p/\delta)n} \|(\theta_1)_{\Theta}\|_1 + A \sqrt{\log(p/\delta)} (\sqrt{n}\|(\theta_1)_{\Theta} - (\hat \theta)_{\Theta}\|_1 + \sqrt{n}\|\hat \theta_{\Theta^C}\|_1  + \sqrt{S}\|X(\theta_1 - \hat \theta)\|_2),
\end{align*}
where $\Theta$ is the support of $\theta_1$, $\Theta^C$ is its complement, and for any vector $u$, $u_{\Theta}$ is the restriction of $u$ to $\Theta$ (same applies to the complement). This implies since $\kappa >2 A$ that with probability $1-\delta$
\begin{align}
&\|X (\theta_1 - \hat \theta)\|_2 \Big( \|X(\theta_1 - \hat \theta)\|_2 - A \sqrt{\log(p/\delta)S} \Big)  + \kappa/2 \sqrt{\log(p/\delta)n} \|\hat \theta_{\Theta^C}\|_1 \nonumber\\ 
&\leq 3\kappa/2 \sqrt{\log(p/\delta)n} \|(\theta_1)_{\Theta} - \hat \theta_{\Theta}\|_1.\label{eq:coco2}
\end{align}

\paragraph{It holds with high probability that : $\|X (\theta_1 - \hat \theta)\|_2 \leq \frac{3\kappa}{\phi} \sqrt{\log(p/\delta)nS}$.} Assume that $\|X(\hat \theta - \theta_1)\|_2^2> 4A^2 S\log(p/\delta)$. In this case, the previous equation implies that with probability $1-\delta$
\begin{align}\label{eq:coco}
 \|X (\theta_1 - \hat \theta)\|_2^2/2  + \kappa/2 \sqrt{\log(p/\delta)n} \|\hat \theta_{\Theta^C}\|_1 \leq 3\kappa/2 \sqrt{\log(p/\delta)n} \|(\theta_1)_{\Theta} - \hat \theta_{\Theta}\|_1,
\end{align}
which implies in particular
\begin{align*}
\|(\theta_1)_{\Theta^C} - \hat \theta_{\Theta^C}\|_1 \leq 3 \|(\theta_1)_{\Theta} - \hat \theta_{\Theta}\|_1,
\end{align*}
since $(\theta_1)_{\Theta^C} = 0$. By the compatibility condition, we thus know that for this vector,
\begin{align*}
\|(\theta_1)_{\Theta} - \hat \theta_{\Theta}\|_1^2 \leq \frac{S\|X (\theta_1 - \hat \theta)\|_2^2}{n\phi^2},
\end{align*}
which implies that with Equation~\eqref{eq:coco} that
\begin{align*}
 \|X (\theta_1 - \hat \theta)\|_2 \leq \frac{\sqrt{3\kappa}}{\phi} \sqrt{\log(p/\delta)S},
\end{align*}
and so in every case, since $\frac{3\kappa}{\phi} >4A^2$, we have with probability $1-\delta$
\begin{align}\label{louba}
 \|X (\theta_1 - \hat \theta)\|_2 \leq \frac{\sqrt{3\kappa}}{\phi} \sqrt{\log(p/\delta)S}.
\end{align}

By combining this with Equation~\eqref{eq:coco2}, we obtain
\begin{align}\label{eq:coco3}
 \|X (\theta_1 - \hat \theta)\|_2^2   + \kappa/2 \sqrt{\log(p/\delta)n} \|\hat \theta_{\Theta^C}\|_1 \leq 3\kappa/2 \sqrt{\log(p/\delta)n} \|(\theta_1)_{\Theta} - \hat \theta_{\Theta}\|_1 + \frac{3\kappa A}{\phi} S\log(p/\delta).
\end{align}

\paragraph{Case 1: $\|(\theta_1)_{\Theta} - \hat \theta_{\Theta}\|_1\leq \frac{30A}{\phi} S\sqrt{\frac{\log(p/\delta)}{n}}$.} This condition implies with Equation~\eqref{eq:coco3} that
\begin{align*}
 &\|X (\theta_1 - \hat \theta)\|_2^2 + \kappa/2 \sqrt{\log(p/\delta)n} \|\theta_1 - \hat \theta\|_1\\ 
&=  \|X (\theta_1 - \hat \theta)\|_2^2 + \kappa/2  \sqrt{\log(p/\delta)n}\|(\theta_1)_{\Theta} - \hat \theta_{\Theta}\|_1 +  \kappa/2  \sqrt{\log(p/\delta)n}\|\hat \theta_{\Theta^C}\|_1 \\
&\leq \kappa/2  \sqrt{\log(p/\delta)n}\|(\theta_1)_{\Theta} - \hat \theta_{\Theta}\|_1\\
&+  3\kappa/2 \sqrt{\log(p/\delta)n} \|(\theta_1)_{\Theta} - \hat \theta_{\Theta}\|_1 + \frac{3\kappa A}{\phi} S\log(p/\delta) \\
&\leq 2\kappa  \sqrt{\log(p/\delta)n}\|(\theta_1)_{\Theta} - \hat \theta_{\Theta}\|_1  + \frac{3\kappa A}{\phi} S\log(p/\delta)\\
&\leq  \frac{6\kappa A}{\phi} S\log(p/\delta).
\end{align*}


\paragraph{Case 2: $\|(\theta_1)_{\Theta} - \hat \theta_{\Theta}\|_1\geq   \frac{30A}{\phi}  S\sqrt{\frac{\log(p/\delta)}{n}}$.} This condition implies with Equation~\eqref{eq:coco3} that
\begin{align*}
\|(\theta_1)_{\Theta^C} - \hat \theta_{\Theta^C}\|_1 \leq 4 \|(\theta_1)_{\Theta} - \hat \theta_{\Theta}\|_1,
\end{align*}
since $(\theta_1)_{\Theta^C} = 0$. By the compatibility condition, we thus know that for this vector,
\begin{align*}
\|(\theta_1)_{\Theta} - \hat \theta_{\Theta}\|_1^2 \leq \frac{S\|X (\theta_1 - \hat \theta)\|_2^2}{n\phi^2},
\end{align*}
which implies by Equation~\eqref{louba} that
\begin{align*}
\|(\theta_1)_{\Theta} - \hat \theta_{\Theta}\|_1 \leq \frac{3\kappa}{\phi^2} S \sqrt{\frac{\log(p/\delta)}{n}},
\end{align*}
and which implies together with Equation~\eqref{louba} that with probability $1-\delta$,
\begin{align*}
 \|X (\theta_1 - \hat \theta)\|_2^2 + \kappa/2 \sqrt{\log(p/\delta)n} \|\theta_1 - \hat \theta\|_1 \leq  \Big(\frac{3\kappa^2}{\phi^2} +  \frac{3\kappa}{\phi} \Big)S\log(p/\delta).
\end{align*}

\paragraph{Conclusion.} From the two above cases, we know that with probability $1-\delta$,
\begin{align*}
 \|X (\theta_1 - \hat \theta)\|_2^2 + \kappa/2 \sqrt{\log(p/\delta)n} \|\theta_1 - \hat \theta\|_1 \leq  \Big(\frac{3\kappa^2}{\phi^2} +  \frac{6\kappa}{\phi}(A+1) \Big)S\log(p/\delta).
\end{align*}
This implies that
\begin{align}\label{eq:beuh}
 \|X (\theta_1 - \hat \theta)\|_2^2  \leq  \Big(\frac{3\kappa^2}{\phi^2} +  \frac{6\kappa}{\phi}(A+1) \Big)S\log(p/\delta),
\end{align}
and 
\begin{align}\label{eq:boulni}
\|\theta_1 - \hat \theta\|_1  \leq  \Big(\frac{6\kappa}{\phi^2} +  \frac{12(A+1)}{\phi} \Big)S \sqrt{\frac{\log(p/\delta)}{n}}.
\end{align}
Let $a = \theta_1 - \hat \theta$, and $a_{(.)}$ be the order version of $|a|$. Note that the last equation implies that $|a_{(j)}| \leq \frac{1}{j}\Big(\frac{6\kappa}{\phi^2} +  \frac{12(A+1)}{\phi} \Big) S \sqrt{\frac{\log(p/\delta)}{n}}$, and thus
\begin{align}
\sum_{j=S+1}^p a_{(j)}^2 &\leq \sum_{j=S+1}^p \frac{1}{j^2}\Bigg[\Big(\frac{6\kappa}{\phi^2} +  \frac{12(A+1)}{\phi} \Big) S \sqrt{\frac{\log(p/\delta)}{n}}\Bigg]^2 \nonumber\\
&\leq 12 \Big(\frac{6\kappa}{\phi^2} +  \frac{12(A+1)}{\phi} \Big)^2 S \frac{\log(p/\delta)}{n}. \label{cocoti}
\end{align}

Assume that $\hat \theta$ is such that 
\begin{align*}
\|(\theta_1)_{U} - (\hat \theta)_{U}\|_2^2 \leq  \Big(\frac{3\kappa^2}{\phi} +  6\kappa(A+1) \Big) \frac{S\log(p/\delta)}{n},
\end{align*}
where $U$ is the support of the $S$ largest entries of the vector $|\theta_1 - \hat \theta|$, i.e.
\begin{align*}
\sum_{j=1}^S a_{(i)}^2 \leq \Big(\frac{3\kappa^2}{\phi} +  6\kappa (A+1)\Big) \frac{S\log(p/\delta)}{n}.
\end{align*}
Combining this with Equation~\eqref{cocoti} implies
\begin{align*}
\sum_{j=1}^p a_{(i)}^2 \leq \Bigg[\Big(\frac{3\kappa^2}{\phi} +  6\kappa(A+1) \Big) + 12 \Big(\frac{6\kappa}{\phi^2} +  \frac{12}{\phi}(A+1) \Big)^2\Bigg] \frac{S\log(p/\delta)}{n},
\end{align*}
which concludes the proof since $\sum_{j=1}^p a_{(i)}^2 = \|\theta_1 - \hat \theta\|_2^2$

Assume now the contrary, i.e.~that $\hat \theta$ is such that 
\begin{align*}
\|(\theta_1)_{U} - (\hat \theta)_{U}\|_2^2 \geq  \Big(\frac{3\kappa^2}{\phi} +  6\kappa(A+1) \Big) \frac{S\log(p/\delta)}{n}.
\end{align*}
 From Equation~\eqref{eq:beuh}, this implies that the compatibility condition is not satisfied for $\hat \theta - \theta_1$, and we thus know that for $U$ since $U$ is a support of size $S$
\begin{align*}
\|(\theta_1)_{U} - \hat \theta_{U^C}\|_2^2  \geq  4 \|(\theta_1)_{U} - \hat \theta_{U}\|_2^2,
\end{align*}
which is equivalent to
\begin{align*}
\sum_{j=1}^p a_{(i)}^2 \leq \sum_{j=S+1}^p a_{(i)}^2.
\end{align*}
From Equation~\eqref{cocoti}, we have
\begin{align*}
\sum_{j=1}^p a_{(i)}^2 \leq \sum_{j=S+1}^p a_{(i)}^2 \leq 12 \Big(\frac{6\kappa}{\phi^2} +  \frac{12}{\phi}(A+1) \Big)^2 S \frac{\log(p/\delta)}{n},
\end{align*}
and this completes the proof by summing these two terms, using the fact that $\|\theta_2\|_2 \leq C\sqrt{\frac{S\log(p/\delta)}{n}}$, and that $\theta = \theta_1 + \theta_2$.

\section{Proof for Theorem~\ref{th:bigalter}}\label{app:2}

\subsection{A related testing problem} Set

\begin{align*}
\tilde \rho_n &\geq \frac{3}{\sqrt{\min(c_m,1)}}\\
&\times\min\Big(\sqrt{C\log(1/\delta)} n^{-1/4}, \sqrt{E}\sqrt{\frac{S_1 log(p/\delta)}{n}}\Big).
&+ \sqrt{\frac{2(C_M+1)}{\min(c_m,1)}}\sqrt{E}\sqrt{\frac{S_0 log(p/\delta)}{n}}.
\end{align*}

We cast the testing problem
\begin{align}\label{test:2}
H_0:  \theta \in  \mathcal S_{S_0}(C,B, \bar p, \delta) \hspace{2mm} vs. \hspace{2mm} H_1: \theta \in \mathcal{\tilde S}_{S_1}(C, B, \bar p, \delta, \tilde \rho_n).
\end{align}

Let $\theta \in \mathcal S_S(C,B,\bar p, \delta)$ for some $S>0$. By definition of $\hat \theta$ and Theorem~\ref{th:adaest} we know that with probability larger than $1-\delta$
\begin{align}
\| \hat \theta - \theta\|_2^2 &\leq  E \frac{ S\log(p/\delta)}{n}. \label{eq:weirdo}
\end{align}

Consider now the vector $(\hat r_i)_{i \in \mathcal D_2}$, where $\mathcal D_2$ is the second half of the data set. By definition, for $i \in \mathcal D_2$,
\begin{align}
(\hat r_i)^2 - \sigma^2 &= \left(\sum_{j=1}^p X_{i,j} (\theta_j - \hat \theta_j) + \epsilon_i\right)^2 \nonumber\\
&= \left(\sum_{j=1}^p X_{i,j} (\theta_j - \hat \theta_j)\right)^2 + 2\left(\sum_{j=1}^p X_{i,j} (\theta_j - \hat \theta_j)\right)\epsilon_i + (\epsilon_i^2 - \sigma^2). \label{eq:res0}
\end{align}
Note first that since the $\epsilon_i$ are independent sub-Gaussian random variables, the elements of $(\hat r_i)_{i \in \{1,...,n\}}$ are independent among each others.

The mean of the first elements in the decomposition in Equation~\eqref{eq:res0} of $\hat r_i$, namely $\frac{1}{n}\|X(\theta - \hat \theta)\|_2^2 = \frac{1}{n} \sum_{i=1}^n\left(\sum_{j=1}^p X_{i,j} (\theta_j - \hat \theta_j)\right)^2$, is such that (by Assumption~\ref{ass:design}, since $\theta - \hat \theta$ is $\bar p$ sparse, see proof of Theorem~\ref{th:adaest})
\begin{align}\label{eq:res1}
c_m \|\theta - \hat \theta\|_2^2 \leq \frac{1}{n}\|X(\theta - \hat \theta)\|_2^2 \leq C_M \|\theta - \hat \theta\|_2^2,
\end{align}
The second term in the decomposition, by remarking that $\Big(\frac{1}{n} \left(\sum_{j=1}^p X_{i,j} (\theta_j - \hat \theta_j^{(1)})\right)\epsilon_i\Big)_i$ is a centred sub-Gaussian, uncorrelated, random process with tail coefficient bounded by $\frac{1}{n}c \|X(\theta - \hat \theta)\|_2 \leq C_m c\|\theta - \hat \theta\|_2$, we obtain by Hoeffding's inequality that with probability at least $1-\delta$,
\begin{align}\label{eq:res2}
\Big|\frac{1}{n} \sum_{i =1}^n 2\left(\sum_{j=1}^p X_{i,j} (\theta_j - \hat \theta_j)\right)\epsilon_i\Big|\leq C' C_Mc \|\theta - \hat \theta\|_2 \sqrt{\frac{\log(1/\delta)}{n}}.
\end{align}
Finally, the last term of the decomposition being also a sub-Gaussian random variable with tail coefficient bounded by $c^2$, we have by Hoeffding's inequality, with probability at least $1-\delta$,
\begin{align}\label{eq:res3}
\Big|\frac{1}{n} \sum_{i =1}^n (\epsilon_i^2 - \sigma^2)\Big|\leq C'c \sqrt{\frac{\log(1/\delta)}{n}}.
\end{align}
Finally, by Equations~\eqref{eq:res0} to~\eqref{eq:res3}, by definition of $R_n$ and with probability larger than $1-4\delta$,
\begin{align}
 &c_m \|\theta - \hat \theta\|_2^2  - C'c (C_M\|\theta - \hat \theta\|_2 + 1) \sqrt{\frac{\log(1/\delta)}{n}}\nonumber\\
&\leq R_n \leq  C_M\|\theta - \hat \theta\|_2^2  + C'c (C_M\|\theta - \hat \theta\|_2 + 1) \sqrt{\frac{\log(1/\delta)}{n}}, \nonumber
\end{align}
which implies 
\begin{align}
 c_m \|\theta - \hat \theta\|_2^2  -C n^{-1/2}\log(1/\delta) \leq R_n \leq  C_M\|\theta - \hat \theta\|_2^2  + C n^{-1/2}\log(1/\delta),  \label{concRn}
\end{align}
since by definition $C \geq 16C' (C_M^2B^2 + c^2)$ and $\|\theta - \hat \theta\|_2^2 \leq 4 B^2$ as both vectors are in $l_2(B)$.

\textbf{Under the null hypothesis $H_0$.} Let $\theta \in \mathcal S_{S_0}(C,B,\bar p, \delta)$. Let us re-order the coefficients of $\theta$ as $\theta_{(1)}^2\geq \theta_{(2)}^2 \geq ...\geq \theta_{(p)}^2$. Equation~\eqref{eq:weirdo} imply that with probability larger than $1-\delta$,
\begin{align*}
\sum_{j=S_0+1}^p \hat \theta_{(j)}^2 \leq E\frac{S_0  \log(p/\delta)}{n} <(\tau_n')^2.
\end{align*}
Also Equations~\eqref{eq:weirdo} and~\eqref{concRn} imply that with probability at least $1-\delta$,
\begin{align*}
 R_n  \leq C n^{-1/2}\log(1/\delta) + C_M E\frac{S_0\log(p/\delta)}{n} \leq \tau_n^2,
\end{align*}
by definition of $\tau_n^2$. This implies that the test is accepted with probability larger than $1-\delta$.

\textbf{Under the alternative hypothesis $H_1$.} Assume now that $\theta \in \mathcal{\tilde S}_{S_1}( C,B,\bar p, \delta, \tilde  \rho_n)$. By definition of the set $\mathcal{\tilde S}_{S_1}(C,B,\bar p, \delta, \tilde \rho_n)$, we know that 
\begin{equation}\label{eq:cas2}
\tilde  \rho_n^2 \geq \|\theta - \mathcal S_{S_0}\|_2^2 \geq \sum_{j=S_0+1}^{p}\theta_{(k)}^2,
\end{equation}
where $\tilde  \rho_n =\frac{1}{\sqrt{\min(c_m,1)}} \Big( \min(3\sqrt{Cn^{-1/2}\log(1/\delta)}, 3\sqrt{E\frac{S_1\log(p/\delta)}{n}}) + 2\sqrt{E\frac{S_0\log(p/\delta)}{n}}\Big)$.

\emph{Case 1: $\tilde \rho_n^2 = \frac{1}{\sqrt{\min(c_m,1)}} \Big(3\sqrt{E\frac{S_1\log(p/\delta)}{n}} + 2\sqrt{E\frac{S_0\log(p/\delta)}{n}}\Big)$.} Equation~\eqref{eq:weirdo} imply that with probability larger than $1-\delta$,
\begin{align*}
&\sum_{j=S_0+1}^p \hat \theta_{(j)}^2 \geq  \frac{1}{\sqrt{\min(c_m,1)}} 3E\frac{S_1\log(p/\delta)}{n} -  E\frac{S_1 n \log(p/\delta)}{n}\\ 
&=  2E\frac{S_1\log(p/\delta)}{n}> 2E\frac{S_0\log(p/\delta)}{n} > (\tau_n')^2,
\end{align*}
since $\min(c_m,1)<1$. This implies that the test is rejected with probability larger than $1-\delta$.

\emph{Case 2: $\tilde \rho_n^2 = \frac{1}{\sqrt{\min(c_m,1)}} \Big(3\sqrt{C n^{-1/2}\log(1/\delta)} + 2\sqrt{E\frac{S_0\log(p/\delta)}{n}}\Big)$.} If $\sum_{j=S_0+1}^p \hat \theta_{(j)}^2 \geq  2E\frac{S_0\log(p/\delta)}{n} > (\tau_n')^2$, then the test is rejected. Assume now that
\begin{equation*}
\sum_{j=S_0+1}^p \hat \theta_{(j)}^2 \leq  2E\frac{S_0\log(p/\delta)}{n}.
\end{equation*}
This implies in particular, since $\sum_{j=S_0+1}^{p}\theta_{(k)}^2 > \sqrt{4E\frac{S_0\log(p/\delta)}{n}} + n^{-1/4}$ that
\begin{align*}
\|\theta - \hat \theta\|_2^2 &> |\frac{1}{\sqrt{\min(c_m,1)}} \Big(3\sqrt{Cn^{-1/2}} + 2\sqrt{E\frac{S_0\log(p/\delta)}{n}}\Big) - \sqrt{2E\frac{S_0\log(p/\delta)}{n}}|^2\\ 
&>  \frac{1}{\min(c_m,1)} |3\sqrt{Cn^{-1/2}}  + \frac{1}{2} \sqrt{E\frac{S_0\log(p/\delta)}{n}}|^2 |^2\\
&>\frac{1}{\min(c_m,1)} \Big(9Cn^{-1/2}  + \frac{1}{4} E\frac{S_0\log(p/\delta)}{n}\Big).
\end{align*}

Equation~\eqref{concRn} and the previous equation imply that with probability at least $1-4\delta$,
\begin{align*}
 R_n  \geq c_m\|\theta - \hat \theta\|_2^2 - C n^{-1/2} \geq 9Cn^{-1/2}  + \frac{1}{4} E\frac{S_0\log(p/\delta)}{n} - C n^{-1/2} \geq \tau_n^2,
\end{align*}
by definition of $\tau_n^2$. This implies that the test is accepted with probability larger than $1-\delta$.

\textbf{Conclusion} Finally, we have by combining the results under $H_1$ and $H_0$ that the test is $2\delta-$uniformly consistent for such a $\rho_n^2$, i.e.~it satisfies
\begin{equation}\label{def:uct}
\sup_{\theta \in \mathcal{S}_{S_0}(C,B, \bar p, \delta)} \mathbb E_{\theta} \Psi_n + \sup_{\theta \in \mathcal{\tilde S}_{S_1}(C,B, \bar p, \delta, \tilde \rho_n) } \mathbb E_{\theta} [1-\Psi_n] \leq 2\delta,
\end{equation}
by just summing the bounds on the probability of error under $H_0$ and $H_1$.

\subsection{The confidence set is honest and adaptive}

\textbf{The confidence set is adaptive and honest on $\mathcal P_n(\tilde \rho_n)$}

Let $\theta \in \mathcal P_n(\tilde \rho_n)$. Let $S=S_0$ if $\theta \in \mathcal S_{S_0}$, and $S=S_1$ otherwise. Let $\hat S = S_0$ is $\Psi_n = 0$ and $\hat S = S_1$ otherwise. We have
\begin{align*}
\mathbb P_{\theta}(\theta \in C_n) &= \mathbb P_{\theta}(\|\hat \theta - \theta\|_2^2 \leq E\frac{\hat S\log(p/\delta)}{n})\\
&\geq \mathbb P_{\theta}(\|\hat \theta - \theta\|_2^2 \leq E\frac{ S\log(p/\delta)}{n}, S = \hat S)\\
\geq 1 - \mathbb P_{\theta}(\|\hat \theta - \theta\|_2^2 \geq E\frac{ S\log(p/\delta)}{n}) - \mathbb P_{\theta}(S \neq \hat S)\\
\geq 1 - 3 \delta,
\end{align*}
by definition of $\hat \theta$ as an adaptive estimate, and by $\Psi_n$ is $2\delta-$uniformly consistent on $\mathcal P_n(\tilde \rho_n)$.

Moreover
\begin{align*}
\mathbb P_{\theta}(|C_n|_2^2 \leq E\frac{ S\log(p/\delta)}{n}) = \mathbb P_{\theta}(\hat S = S) \geq 1 - 2\delta,
\end{align*}
$\Psi_n$ is $2\delta-$uniformly consistent on $\mathcal P_n(\tilde \rho_n)$.

The confidence set is thus $3\delta-$adaptive and honest on $\mathcal P_n(\tilde \rho_n)$. This concludes the proof in cases (i) and (ii) of Figure~\ref{fig:1} since in these cases $\rho_n \geq \tilde \rho_n$.

\textbf{Case (iii) of Figure~\ref{fig:1} ($S_0\geq n^{1/2}/\log(p)$ and $\rho_n = 0$)}

Consider now $\theta \in \mathcal S_{S_1}(C,B,\bar p, \delta) \setminus \mathcal P_n(\tilde \rho_n)$. By definition,
$$\|\theta - \mathcal S_{S_0}\|_2^2\leq \tilde \rho_n \leq \sqrt{Cn^{-1/2}\log(1/\delta)}+ 2\sqrt{E\frac{S_0\log(p/\delta)}{n}}\leq (\sqrt{C}+2\sqrt{E})\sqrt{\frac{S_0\log(p/\delta)}{n}},$$
since $S_0 \geq n^{1/2}\log(1/\delta)/\log(p/\delta)$. This implies that $\theta \in \mathcal S_{S_0}((\sqrt{C}+2\sqrt{E})^2,B,\bar p, \delta)$, and we have since $\hat \theta$ is an adaptive estimate that with probability $1-\delta$
$$\|\hat \theta - \theta\|_2^2 \leq E\frac{ S\log(p/\delta)}{n}.$$

So in this case
\begin{align*}
\mathbb P_{\theta}(\theta \in C_n) &= \mathbb P_{\theta}(\|\hat \theta - \theta\|_2^2 \leq E\frac{\hat S\log(p/\delta)}{n})\\
&\geq \mathbb P_{\theta}(\|\hat \theta - \theta\|_2^2 \leq E\frac{ S_0\log(p/\delta)}{n}, S = \hat S)\\
&\geq 1-\delta.
\end{align*}
So the confidence interval is honest and adaptive for these parameters also, and thus it is honest and adaptive for the entire set $ \mathcal S_{S_1}(C,B,\bar p, \delta)$ (i.e.~$\rho_n =0$). This also concludes the proof in this case.

\section{Proof of existence of an estimate in the general case}\label{app:1}

\begin{theorem}\label{th:adaest2}
Let Assumptions~\ref{ass:noise}, and~\ref{ass:design} be satisfied for some $c, c_m>0, c_M, 2\bar p >0$. Let $B>0$ and $C>0$. Let $\delta>0$. There exists an estimate $\hat \theta$ of $\theta$ that is such that for some universal constants $D,G>0$ that depends only on $C,B,c, c_m, C_M$, we have
\begin{equation*}
\forall 0< S \leq  G \bar p, \sup_{\theta \in \mathcal S_S(C,B, \bar p, \delta) }\hspace{-1mm}\mathbb P_{\theta}\Big( \|\hat \theta - \theta\|_2^2 \geq \frac{E S  \log(p/\delta)}{n}\Big) \leq \delta.
\end{equation*}
\end{theorem}
\begin{proof}

\textbf{Step 1: Decomposition of the problem}

Let $\theta \in \mathcal S_S(C,B,\bar p,\delta)$. Let $\theta = \theta_1 + \theta_2$, where $\theta_1$ contains the $S-$largest components of $\theta$, and $\theta_2$ the rest. We have for any estimator $\hat \theta$
\begin{align}
\|Y - Xu\|_2^2  &= \|\epsilon + X \theta_2\|_2^2 + 2\langle \epsilon + X\theta_2, X(\theta_1 - u) \rangle +\|X (\theta_1 - u)\|_2^2 \nonumber\\
&= \upsilon^2 + 2\langle \epsilon, X(\theta_1 - u) \rangle +  2\langle X\theta_2, X(\theta_1 - u) \rangle  +\|X (\theta_1 - u)\|_2^2 ,\label{eq:cici1}
\end{align}
where $\upsilon^2 = \|\epsilon + X \theta_2\|_2^2$ does not depend on $u$.

Note first that
\begin{align}
|\langle  X\theta_2, X(\theta_1 - u) \rangle| &\leq \|X\theta_2\|_2 \|X(\theta_1 - u) \|_2 \nonumber\\
&\leq C_M\sqrt{n}\|\theta_2\|_2 \|X(\theta_1 - u) \|_2 \nonumber\\
&\leq  C_M\sqrt{C S\log(p/\delta)}\|X(\theta_1 - u) \|_2, \label{eq:cici2}
\end{align}
see Equation~\eqref{eq:acici2}.

Let $k>0$. In the proof of Lemma 2 in the paper~\cite{nickl2012confidence} (third equation on P 19), we have that for any $b>0$ there exists a constant $E>0$ such that for any $u>0$
\begin{align*}
\mathbb P\Big(\sup_{u: u \in l_0(k), \|X(u - \theta_1)\|_2 \leq b} |\langle  \epsilon, X(\theta_1 - u) \rangle| \geq Eb\sqrt{2k \log(p)} + b\sqrt{2u}\Big)\leq \exp(-u).
\end{align*}
The proof of this equation in the paper~\cite{nickl2012confidence} is based on the fact that the $e-$covering number of the set $\{u: u \in l_0(k), \|X(u - \theta_1)\|_2 = b\}$ is upper bounded by $\Big(p \big(\frac{2b+e}{e}\big)\Big)^{k+1}$ and then on techniques that feature mainly Dudley's entropy bound, and Borell Sudakov Cirelson concentration inequality (or also Talagrand's inequality). 

Last inequality implies that for any $\delta>0$, we have that on some event $\xi_{k,b}(\delta)$ of probability larger than $1-\delta$ (setting $u = \log(1/\delta)$)
\begin{align}\label{eq:booz}
\sup_{u: u \in l_0(k), \|X(u - \theta_1)\|_2 \leq b} |\langle  \epsilon, X(\theta_1 - u) \rangle| \leq Eb\sqrt{2k \log(p)} + b\sqrt{2\log(1/\delta)}.
\end{align}
Let now $v$ be such that $v \in l_0(k), \|X(v - \theta_1)\|_2 > b$. Let $c>1$ be such that $\|X(v - \theta_1)\|_2 = cb$. Since $v$ is $k$-sparse, and $\theta_1$ is $k-$sparse, there exists $u$ such that $u \in l_0(k+S)$, $u-\theta_1=  \frac{1}{c}(v - \theta_1)$ and $\|X(u-\theta_1)\|_2^2 = \frac{1}{c^2}\|X(v - \theta_1)\|_2^2=b^2$. Also by Equation~\eqref{eq:booz} (applied in $k+S$), we know that for this $u$, on the event $\xi_{k+S,b}(\delta)$ where Equation~\eqref{eq:booz} (applied in $k+S$) is satisfied
\begin{align*}
|\langle  \epsilon, X(u-\theta_1) \rangle| \leq Eb\sqrt{2(k+S) \log(p)} + b\sqrt{2\log(1/\delta)}.
\end{align*}
So since $u-\theta_1=  c(v - \theta_1)$, we have also that on the event $\xi_{k+S,b}(\delta)$
\begin{align*}
|\langle  \epsilon, X(v-\theta_1) \rangle| \leq Ecb\sqrt{2(k+S) \log(p)} + cb\sqrt{2\log(1/\delta)}.
\end{align*}
This implies on particular that on the event $\xi_{k+S,b}(\delta)$ of probability larger than $1-\delta$ 
\begin{align*}
\sup_{c\geq 1}\sup_{u: u \in l_0(k), \|X(u - \theta_1)\|_2 \leq cb} \frac{|\langle  \epsilon, X(\theta_1 - u) \rangle|}{c} \leq Eb\sqrt{2(k+S) \log(p)} + b\sqrt{2\log(1/\delta)},
\end{align*}
and this holds for any $b>0$ with $E$ that depends on $b$. Setting $b=1$ and considering the associated $E$, this implies that with probability larger than $1-\delta$ 
\begin{align*}
\sup_{b>1}\sup_{u: u \in l_0(k), \|X(u - \theta_1)\|_2 \leq b} \frac{|\langle  \epsilon, X(\theta_1 - u) \rangle|}{b} \leq C\sqrt{2(k+S) \log(p)} + \sqrt{2\log(1/\delta)},
\end{align*}
and by an union bound on all $k \leq p$, we get that with probability larger than $1-\delta$
\begin{align*}
\sup_{b>1, k \leq p}\sup_{u: u \in l_0(k), \|X(u - \theta_1)\|_2 \leq b} \frac{|\langle  \epsilon, X(\theta_1 - u) \rangle|}{b\sqrt{k+S}} \leq C\sqrt{2\log(p)} + \sqrt{2\log(p/\delta)/(k+S)} \leq D\sqrt{\log(p/\delta)}.
\end{align*}
where $D$ is some universal constant. This implies that with probability larger than $1-\delta$, for any $u \in \mathbb R^p$, we have
\begin{align}\label{eq:cici3}
|\langle  \epsilon, X(\theta_1 - u) \rangle|     \leq  D (\|X(\theta_1 - u)\|_2 + 1) \sqrt{(k+S)\log(p/\delta)}.
\end{align}

Equations~\eqref{eq:cici1}, \eqref{eq:cici2} and~\eqref{eq:cici3} imply that with probability larger than $1-\delta$,
\begin{align}
\|Y - X \hat \theta\|_2^2  &\geq
 \upsilon^2 +  \|X (\theta_1 - \hat \theta)\|_2^2  - 2 D (\|X(\theta_1 - u)\|_2 + 1) \sqrt{(k+S)\log(p/\delta)} \nonumber\\ 
&- 2C_M\sqrt{C S\log(p/\delta)}\|X(\theta_1 - u) \|_2\nonumber\\
&\geq \upsilon^2 +  \|X (\theta_1 - \hat \theta)\|_2^2  - A (\|X(\theta_1 - u)\|_2 + 1) \sqrt{(k+S)\log(p/\delta)},\label{eq:victo}
\end{align}
where $A:=2 (D + C_M\sqrt{C})$.

\textbf{Step 2: Definition of the estimator}

Define now the estimator $\hat \theta$ as being
\begin{align}
\hat \theta &= \arg\min_u \Big[ \|Y - Xu\|_2^2 + \kappa \log(p/\delta) \|u\|_0\Big],\label{min0}
\end{align}
where $\kappa>2A(F+1)$, where $F= \Big(A (\sqrt{2} + 4A + 2\sqrt{A} + 1) + 1\Big) $. Let $k$ be the sparsity of $\hat \theta$ from now on.



Since $\hat \theta$ is the minimizer of the above formula, it is in particular such that (with $u=\theta_1$)
\begin{align*}
\|Y - X \hat \theta\|_2^2 +   \kappa  \log(p/\delta) \|\hat \theta\|_0 &\leq  \|Y - X \theta_1\|_2^2 +   \kappa  \log(p/\delta) \| \theta_1\|_0,
\end{align*}
which implies together with Equations~\eqref{eq:cici1} and~\eqref{eq:victo}
\begin{align}
\|X (\theta_1 - \hat \theta)\|_2^2  - A (\|X(\theta_1 - u)\|_2 + 1) \sqrt{(k+S)\log(p/\delta)} +  \kappa   \log(p/\delta) k &\leq  \kappa   \log(p/\delta) \| \theta_1\|_0.\label{eq:victo2}
\end{align}

\textbf{Step 3: Proof that with high probability, $\|X(\theta_1 - \hat \theta)\|_2 \leq  (\sqrt{2} + 4A + 2\sqrt{A}) \sqrt{(k+S)\log(p/\delta)}$.}

Assume that $\hat \theta$ is such that $\|X(\theta_1 - \hat \theta)\|_2^2 >  2A (\|X(\theta_1 - u)\|_2 + 1) \sqrt{(k+S)\log(p/\delta)} $.

Equation~\eqref{eq:victo2} implies that
\begin{align*}
\|X (\theta_1 - \hat \theta)\|_2^2/2 +    \log(p/\delta) \|\hat \theta\|_0 &\leq    \log(p/\delta) \| \theta_1\|_0,
\end{align*}
which implies that $\|\hat \theta\|_0 \leq  \| \theta_1\|_0 \leq S$. By the previous equation this implies
\begin{align*}
\|X(\theta_1 - \hat \theta)\|_2 &\leq   \sqrt{2S \log(p/\delta)}.
\end{align*}

This implies that in any case (also when $\|X(\theta_1 - \hat \theta)\|_2^2 \leq  2A (\|X(\theta_1 - u)\|_2 + 1) \sqrt{(k+S)\log(p/\delta)} $)
\begin{align}\label{eq:victo3}
\|X(\theta_1 - \hat \theta)\|_2 &\leq  (\sqrt{2} + 4A + 2\sqrt{A} + 1) \sqrt{(k+S)\log(p/\delta)} = F\sqrt{(k+S)\log(p/\delta)}.
\end{align}

\textbf{Step 4: Bound on the sparsity $k$ of $\hat \theta$.}

From Equations~\eqref{eq:victo2} and~\eqref{eq:victo3}, we deduce that
\begin{align*}
 \kappa \log(p/\delta) (k-S) &\leq  A (F+1) (k+S)\log(p/\delta).
\end{align*}
Since $\kappa > 2 A (F+1)$, it implies that
\begin{align*}
k &\leq  \big(A(F+1) + \kappa \big)S := GS.
\end{align*}

\textbf{Step 4: Conclusion.}

Since $k \leq GS \leq \bar p$, Assumption~\ref{ass:design} applies to $\theta_1 - \hat \theta$
\begin{align*}
\|X(\theta_1 - \hat \theta)\|_2 &\geq  c_m\sqrt{n} \|\theta_1 - \hat \theta\|_2.
\end{align*}
This implies together with Equation~\eqref{eq:victo3} that with probability larger than $1-\delta$
\begin{align*}
\|\theta_1 - \hat \theta\|_2 &\leq  \frac{1}{c_m}(\sqrt{2} + 4A + 2\sqrt{A} + 1) \sqrt{\frac{S\log(p/\delta)}{n}}.
\end{align*}
This concludes the proof using the fact that $\|\theta_2\|_2 \leq C\sqrt{\frac{S\log(p/\delta)}{n}}$, and that $\theta = \theta_1 + \theta_2$.

\end{proof}


\begin{thebibliography}{27}
\providecommand{\natexlab}[1]{#1}
\providecommand{\url}[1]{\texttt{#1}}
\expandafter\ifx\csname urlstyle\endcsname\relax
  \providecommand{\doi}[1]{doi: #1}\else
  \providecommand{\doi}{doi: \begingroup \urlstyle{rm}\Url}\fi

\vspace{-0.4cm}
 \footnotesize{


\bibitem[Abbasi-Yadkori et.~al.(2012)]{abbasi2012}
Y.~Abbasi-Yadkori, D.~Pal, C.~Szepesvari.
\newblock {Online-to-Confidence-Set Conversions and Application to Sparse Stochastic Bandits}.
\newblock {\em Artificial Intelligence and Statistics}, 2012.




\bibitem[Baraud(2004)]{baraud2004confidence}
Y.~Baraud.
\newblock Confidence balls in gaussian regression.
\newblock \emph{Annals of statistics}, 528--551, 2004.




\bibitem[Beran and D\"umbgen(1998)]{beran1998modulation}
R.~Beran and L.~D\"umbgen.
\newblock Modulation of estimators and confidence sets.
\newblock \emph{ Annals of Statistics}, 1826--1856,1998.


\bibitem[Bickel et.~al.(2009)]{bickel2009simultaneous}
P.~Bickel, Y.~Ritov, A.~Tsybakov.
\newblock Simultaneous analysis of Lasso and Dantzig selector.
\newblock \emph{ Annals of Statistics}, 37\penalty0 (4):\penalty0 1705--1732, 2009.



\bibitem[B\"uhlmann and van de Geer(2011)]{buhlmann2011statistics}
P.~B\"uhlmann and S.A.~van de Geer.
\newblock Statistics for high-dimensional data: : Methods, Theory and Applications.
\newblock \emph{Springer}, 2011.





\bibitem[Cand{\`e}s et.~al.(2006)]{candes2006robust}
E.J. Cand{\`e}s, J.~Romberg, and T.~Tao.
\newblock {Robust uncertainty principles: Exact signal reconstruction from
  highly incomplete frequency information}.
\newblock {\em IEEE Transactions on information theory}, 52(2):489--509, 2006.




\bibitem[Cand{\`e}s and Tao(2007)]{candes2007dantzig}
E.~Candes and T.~Tao.
\newblock {The Dantzig selector: statistical estimation when p is much larger
  than n}.
\newblock {\em Annals of Statistics}, 35(6):2313--2351, 2007.


\bibitem[Cand{\`e}s(2008)]{candes2008restricted}
E.J. Cand{\`e}s.
\newblock {The restricted isometry property and its implications for compressed
  sensing}.
\newblock {\em Comptes Rendus Mathematique}, 346(9-10):589--592, 2008.




\bibitem[Carpentier and Munos(2012)]{carpentier2012}
A.~Carpentier and R.~Munos.
\newblock {Bandit theory meets compressed sensing for high dimensional stochastic linear bandit}.
\newblock {\em Artificial Intelligence and Statistics}, 2012.



\bibitem[Desphandes and Montanari(2012)]{desphande2012}
Y.~Desphandes, and A.~Montanari.
\newblock  Linear Bandits in High Dimension and Recommendation Systems.
\newblock \emph{arXiv preprint}, 2013.



\bibitem[Donoho and Stark(1989)]{donoho1989uncertainty}
D.L.~Donoho and P.B.~ Stark.
\newblock {Uncertainty principles and signal recovery}.
\newblock {\em SIAM Journal on Applied Mathematics}, 49(3):906--931, 1989.




\bibitem[Donoho(2006)]{donoho2006compressed}
D.L.~Donoho.
\newblock {Compressed sensing}.
\newblock {\em IEEE Transactions on Information Theory}, 52(4):1289--1306,
  2006.

\bibitem[Foucart and Lai(2009)]{foucart2009sparsest}
S.~Foucart and M.J.~Lai.
\newblock {Sparsest solutions of underdetermined linear systems via lq-minimization for $0<q<p$}.
\newblock {\em Applied and Computational Harmonic Analysis}, 26(3):395--407, 2009.


\bibitem[Gin{\'e} and Nickl(2010)]{gine2010confidence}
E.~Gin{\'e} and R.~Nickl.
\newblock Confidence bands in density estimation.
\newblock \emph{Annals of Statistics}, 38\penalty0 (2):\penalty0
  1122--1170, 2010.


\bibitem[Hoffmann and Lepski(2002)]{hoffman2002random}
M.~Hoffmann and O.~Lepski.
\newblock Random rates in anisotropic regression.
\newblock \emph{Annals of statistics},38\penalty0 (2):\penalty0 325--358, 2002.






\bibitem[Juditsky and Lambert-Lacroix(2003)]{juditsky2003nonparametric}
A.~Juditsky and S.~Lambert-Lacroix.
\newblock Nonparametric confidence set estimation.
\newblock \emph{Mathematical Methods of Statistics}, 12\penalty0 (4):\penalty0
  410--428, 2003.


\bibitem[Javanmard and Montanari(2013)]{javanmard2013confidence}
A.~Javanmard, and A.~Montanari.
\newblock Confidence Intervals and Hypothesis Testing for High-Dimensional statistical model.
\newblock \emph{Advances in Neural Information Processing System}, 1187--1195, 2013.




\bibitem[Javanmard and Montanari(2014)]{javanmard2014confidence}
A.~Javanmard, and A.~Montanari.
\newblock Confidence intervals and hypothesis testing for high-dimensional regression.
\newblock \emph{Journal of Machine Learning research}, to appear, 2014.





\bibitem[Kavukcuoglu et.~al.(2009)]{kavukcuoglu2009learning}
K.~Kavukcuoglu, M.~Ranzato, R.~Fergus and Y.~Le-Cun.
\newblock Learning invariant features through topographic filter maps.
\newblock \emph{IEEE Conference on Computer Vision and Pattern Recognition.}, 1605--1612, 2009.


\bibitem[Koltchinskii(2009)]{koltchinskii2009dantzig}
V.~Koltchinskii.
\newblock {The Dantzig selector and sparsity oracle inequalities}.
\newblock {\em Bernoulli}, 15(3):799--828, 2009.




\bibitem[Lee et.~al.(2013)]{lee2013exact}
Lee, Jason D and Sun, Dennis L and Sun, Yuekai and Taylor, Jonathan E
\newblock Exact inference after model selection via the Lasso.
\newblock arXiv preprint arXiv:1311.6238, 2013.

\bibitem[Lee et.~al.(2014)]{lee2014}
Lee, Jason D and Taylor, Jonathan E
\newblock Exact Post Model Selection Inference for Marginal Screening.
\newblock arXiv preprint arXiv:1402.5596, 2014.


\bibitem[Moriaka and Satoh(2010)]{NIPS2011_0141}
N.~Morioka and S.~Satoh.
\newblock Generalized Lasso based Approximation of Sparse Coding for Visual Recognition.
\newblock \emph{Advances in Neural Information Processing Systems 24}, J.~Shawe-Taylor and R.S.~Zemel and P.~Bartlett and F.C.N.~Pereira and K.Q.~Weinberger, 181--189, 2011.

\bibitem[Nickl and van de Geer(2013)]{nickl2012confidence}
R.~Nickl and S.~van de Geer.
\newblock Confidence Sets in Sparse Regression.
\newblock \emph{Annals of Statistics}, 41\penalty0 (6):\penalty0
  2852--2876, 2013.


\bibitem[Raskutti et al(2011)]{raskutti2011minimax}
Raskutti, Garvesh and Wainwright, Martin J and Yu, Bin.
\newblock Minimax Rates of Estimation for High-Dimensional Linear Regression Over $l_q$ balls.
\newblock \emph{Information Theory, IEEE Transactions on}, vol 57, number 10 pages 6979-6994, 2011.



\bibitem[Starck et.~al.(2010)]{starck2010sparse}
J.~-L.~Starck, F.~Murtagh and J.~Fadili.
\newblock  Sparse image and signal processing: wavelets, curvelets, morphological diversity.
\newblock \emph{Cambridge University Press}, 2010.



\bibitem[Tibshirani(1994)]{Tibshirani94regressionshrinkage}
R.~Tibshirani.
\newblock Regression shrinkage and selection via the {L}asso.
\newblock {\em Journal of the Royal Statistical Society, Series B},
  58:267--288, 1994.




\bibitem[van de Geer and B\"uhlmann(2014)]{VandeGeer}
S.A.~van de Geer and P.~B\"uhlmann .
\newblock On asymptotically optimal confidence regions and tests for high-dimensional models.
\newblock \emph{ Annals of Statistics}, 42\penalty0 (3):\penalty0 1166--1202, 2014.





\bibitem[Zou(2006)]{zou2006adaptive}
H.~Zou.
\newblock The adaptive lasso and its oracle properties.
\newblock \emph{Journal of the American statistical association}, 101\penalty0 476:\penalty0
  1418--1429, 2006.



}



\end{thebibliography}
\end{document}